\documentclass{article}
\pdfoutput=1

\usepackage{times}
\usepackage{graphicx}
\usepackage{subfig}

\usepackage{amsthm}
\usepackage{amsmath,amssymb}
\usepackage{algorithm}
\usepackage{algorithmicx}
\usepackage{algpseudocode}
\usepackage[breaklinks=true,bookmarks=false,nolinks=true]{hyperref}
\usepackage[nolist,nohyperlinks]{acronym}
\usepackage{color}
\usepackage{bbm}
\usepackage{bold-extra}
\usepackage{anysize}
\usepackage{authblk}
\usepackage{enumerate}

\newcommand{\bA}{\boldsymbol{A}}

\newcommand{\bG}{\boldsymbol{G}}

\newcommand{\bK}{\boldsymbol{K}}
\newcommand{\ba}{\boldsymbol{a}}
\newcommand{\bb}{\boldsymbol{b}}
\newcommand{\bhatb}{\boldsymbol{\hat{b}}}
\newcommand{\bx}{\boldsymbol{x}}

\newcommand{\bX}{\boldsymbol{X}}

\newcommand{\bI}{\boldsymbol{I}}

\newcommand{\by}{\boldsymbol{y}}

\newcommand{\sY}{\mathcal{Y}}

\newcommand{\bhatX}{\boldsymbol{\hat{X}}}

\newcommand{\bz}{\boldsymbol{z}}
\newcommand{\bZ}{\boldsymbol{Z}}
\newcommand{\bC}{\boldsymbol{C}}
\newcommand{\bhatC}{\boldsymbol{\hat{C}}}

\newcommand{\bw}{\boldsymbol{w}}
\newcommand{\bhatw}{\boldsymbol{\hat{w}}}

\newcommand{\bzero}{\boldsymbol{0}}
\newcommand{\balpha}{\boldsymbol{\alpha}}

\newcommand{\sL}{\mathcal{L}}
\newcommand{\sX}{\mathcal{X}}

\newcommand{\sF}{\mathcal{F}}
\newcommand{\sH}{\mathcal{H}}

\DeclareMathOperator*{\argmin}{arg\,min}
\DeclareMathOperator*{\argmax}{arg\,max}

\newcommand{\supp}{\textnormal{supp}}

\newcommand{\field}[1]{\mathbb{#1}}

\newcommand{\R}{\field{R}}

\newcommand{\E}{\field{E}}

\newcommand{\scO}{\mathcal{O}}

\newcommand{\bsigma}{\boldsymbol{\sigma}}

\newcommand{\biota}{\boldsymbol{\iota}}
\newcommand{\spin}{\{-1,+1\}}

\newtheorem{lemma}{Lemma}
\newtheorem{theorem}{Theorem}
\newtheorem{cor}{Corollary}

\newtheorem{prop}{Proposition}

\newcommand{\reals}{\mathbb{R}}

\newcommand{\bbeta}{\boldsymbol{\beta}}

\newcommand{\bhatbeta}{\boldsymbol{\hat{\beta}}}
\newcommand{\tp}{^{\top}}

\newcommand{\MEbr}[2]{\underset{#1}{\mathbb{E}}\left[ #2 \right]}

\newcommand{\Riskh}{\hat{R}}

\newcommand{\Arh}{\hat{A}^\lambda}

\newcommand{\Risk}{R}

\newcommand{\Rad}{\mathfrak{R}}

\newcommand{\Radh}{\hat{\mathfrak{R}}}

\newcommand{\hsrc}{h^{\text{src}}}

\newcommand{\bhsrc}{\boldsymbol{h^{\text{src}}}}

\newcommand{\src}[1]{#1^\text{src}}
\newcommand{\htrg}{h^{\text{trg}}}

\newcommand{\scT}{\hbox{\scshape{t}}}
\newcommand{\OPT}{\hbox{\scshape{opt}}}

\newtheoremstyle{named}{}{}{\itshape}{}{\bfseries}{}{.5em}{\thmnote{#3.}#1}
\theoremstyle{named}
\newtheorem*{nameddef}{}

\begin{acronym}
\acro{IC}{Improvement Condition}
\acro{RLS}{Regularized Least Squares}
\acro{HTL}{Hypothesis Transfer Learning}
\acro{ERM}{Empirical Risk Minimization}
\acro{TEAM}{Target Empirical Accuracy Maximization}
\acro{EAM}{Empirical Accuracy Maximization}
\acro{RKHS}{Reproducing kernel Hilbert space}
\acro{DA}{Domain Adaptation}
\acro{LOO}{Leave-One-Out}
\acro{HP}{High Probability}
\acro{RSS}{Regularized Subset Selection}
\acro{FR}{Forward Regression}
\end{acronym}

\makeatletter

\newtheorem*{rep@theorem}{\rep@title}
\newcommand{\newreptheorem}[2]{%
\newenvironment{rep#1}[1]{%
 \def\rep@title{#2 \ref{##1}.}%
 \begin{rep@theorem}}%
 {\end{rep@theorem}}}
\makeatother

\newreptheorem{theorem}{Theorem}
\newreptheorem{lemma}{Lemma}
\newreptheorem{corollary}{Corollary}

\marginsize{3cm}{3cm}{3cm}{3cm}

\title{Scalable Greedy Algorithms for Transfer Learning}

\begin{document}

\author[a,b,c]{Ilja Kuzborskij\thanks{ilja.kuzborskij@idiap.ch}}
\author[d]{Francesco Orabona\thanks{francesco@orabona.com}}
\author[c,a]{Barbara Caputo\thanks{caputo@dis.uniroma1.it}}

\affil[a]{Idiap Research Institute,
  Centre du Parc, Rue Marconi 19,
  1920 Martigny, Switzerland}
\affil[b]{\'{E}cole Polytechnique F\'{e}d\'{e}rale de Lausanne (EPFL), Switzerland}
\affil[d]{Yahoo Labs,
  229 West 43rd Street,
  10036 New York, NY, USA}
\affil[c]{University of Rome La Sapienza, Dept. of Computer,
  Control and Management Engineering, Rome, Italy}

\maketitle

\begin{abstract}
In this paper we consider the binary transfer learning problem,
focusing on how to select and combine sources from a large pool to yield a good performance on a target task.
Constraining our scenario to real world, we do not assume the direct access to the source data, but rather we employ the source hypotheses trained from them.
We propose an efficient algorithm that selects relevant source hypotheses and feature dimensions simultaneously, building on the literature on the best subset selection problem.
Our algorithm achieves state-of-the-art results on three computer vision datasets, substantially outperforming both transfer learning and popular feature selection baselines in a small-sample setting.
We also present a randomized variant that achieves the same results with the computational cost independent from the number of source hypotheses and feature dimensions.
Also, we theoretically prove that, under reasonable assumptions on the source hypotheses, our algorithm can learn effectively from few examples.

\end{abstract}

\section{Introduction}
%
%
Over the last few years, the visual recognition research landscape has been heavily dominated by Convolutional Neural Networks, thanks to their ability to leverage effectively over massime amount of training data \cite{decaf}. This trend dramatically confirms the widely accepted truth that any learning algorithm performs better when trained on a lot of data.
This
is even more true  
when facing  noisy or ``hard'' problems such as large-scale recognition~\cite{deng2009imagenet}. However,
when tackling large scale recognition problems, gathering substantial training data for all classes considered might be challenging, if not almost impossible.
The occurrence of real-world objects follows a long tail distribution,
with few objects occurring very often, and many with few instances.
Hence, for the vast majority of visual categories known to human beings, it is extremely challenging to collect training data of the order of $10^4-10^5$ instances. 
The ``long tail'' distribution problem was noted and studied by Salakhutdinov~\emph{et~al.}~\cite{salakhutdinov2011learning}, who proposed to address it by leveraging on the prior knowledge available to the learner.
Indeed, learning systems are often not trained from scratch: usually they can be build on previous knowledge acquired
over time on related tasks~\cite{pan2010survey}. 
The scenario of learning from few examples by \emph{transferring} from what is already known to the learner is collectively known as Transfer Learning.
The target domain usually indicates the task at hand and the source domain the prior knowledge of the learner.

Most of the transfer learning algorithms proposed in the recent years focus on the object detection task (binary transfer learning), assuming access to the training data coming from both source and target domains~\cite{pan2010survey}.
While featuring good practical performance~\cite{gong2012geodesic},
they often demonstrate poor scalability w.r.t. the number of sources.
%
%
An alternative direction, 
known as a \ac{HTL}~\cite{kuzborskij2013stability,Ben-DavidU13},
consists in transferring from the \emph{source hypotheses}, that is classifiers trained from them.
This framework is practically very attractive~\cite{aytar2011tabula,tommasi2013learning,kuzborskij2013from}, 
as it treats source hypotheses as black boxes without any regard of their inner workings.

The goal of this paper is to develop an \ac{HTL} algorithm able to deal effectively and efficiently with a large number of sources, where our working definition of large is at least $10^3$. Note that this order of magnitude is also the current frontier in visual classification \cite{deng2009imagenet}.
To this end, we cast Hypothesis Transfer Learning as a problem of \emph{efficient selection} and \emph{combination} of source hypotheses from a large pool.
We pose it 
as a subset selection problem
building on
results from the 
literature 
~\cite{das2008algorithms,zhang2008adaptive}.
We present\footnote{We build upon preliminary results presented in~\cite{kuzborskij2015transfer}.} a greedy algorithm, \verb!GreedyTL!, which attains state of the art performance even with a very limited amount of data from the target domain. 
Morever, we also present a randomized approximate variant of \verb!GreedyTL!, called \verb!GreedyTL-59!, that has a complexity
\emph{independent} from
the number of sources, with no loss in performance.
Our key contribution is a $L2$-regularized variant of the Forward Regression algorithm~\cite{hastie2009elements}.
Since our algorithm can be viewed as a feature selection algorithm as well as an hypothesis transfer learning approach, we extensively evaluate it against popular feature selection and transfer learning baselines.
We empirically demonstrate that \verb!GreedyTL! dominates all the baselines in most small-sample transfer learning scenarios, thus proving the critical role of regularization in our formulation.
Experiments over three datasets show the power of our approach:
we obtain state of the art results in tasks with up to $1000$ classes, totalling $1.2$ million examples, with only $11$ to $20$ training examples from the target domain. 
We back our experimental results by proving generalization bounds showing that, under reasonable assumptions on the source hypotheses, our algorithm is able to learn effectively with very limited data. 

The rest of the paper is organised as follows: after a review of the relevant literature in the field (section \ref{rel-work}), we cast the transfer learning problem in the subset selection framework (section \ref{tl}). We then define our \verb!GreedyTL!, in section \ref{greedy-tl}, deriving its formulation, analysing its computational complexity and its theoretical properties. Section \ref{sec:exps} describes our experimental evaluation and discuss the related findings. We conclude with an overall discussion and presenting possible future research avenues.

\section{Related Work}
\label{rel-work}
%
The problem of how to exploit prior knowledge when attempting to solve a new task with limited, if any, annotated samples is vastly researched. 
Previous work span from  transfer learning \cite{pan2010survey}  to domain adaptation \cite{saenko2010adapting,ben2010theory}, and dataset bias \cite{Efros_2011_6976}. Here we focus on the first. 
In the literature there are several transfer learning settings~\cite{ben2010theory,saenko2010adapting,gong2012geodesic}.
The oldest and most popular is
the one assuming
access to the data originating from both the source and the target domains~\cite{ben2010theory,gong2012geodesic,saenko2010adapting,duan2012exploiting,seah2011healing,tommasi2013frustratingly,kuzborskij2016when}.
There, one typically assumes that plenty of source data are available, but
access to the target data is limited: for instance, we can have many unlabeled examples and only few labeled~\cite{chelappa2014da}.
Here 
we focus on 
the Hypothesis Transfer Learning framework (HTL, ~\cite{kuzborskij2013stability,Ben-DavidU13}).
It 
requires to have access 
only to
\emph{source hypotheses}, that is classifiers or regressors trained on the source domains.
No assumptions are made on how these source hypotheses are trained,
or about their inner workings: they are treated as ``black boxes'', in spirit similar to 
classifier-generated visual descriptors such as Classemes~\cite{classemes} or Object-Bank~\cite{li2010object}.
Several works proposed HTL for visual learning~\cite{aytar2011tabula,tommasi2013learning,oquab2014learning}, some exploiting more explicitly the connection with classemes-like approaches~\cite{jie2011multiclass,cvpr14l2l}, demonstrating an intriguing potential.
Although offering scalability,
HTL-based approaches proposed so far have been tested on problems with less than a few hundred of sources \cite{tommasi2013learning}, already showing some difficulties in selecting informative sources.

Recently, the growing need to deal with large data collections~\cite{deng2009imagenet,choi2010exploiting} has started to change the focus and challenges of research in transfer learning.
Scalability with respect to the amount of data and the ability to identify and separate
informative sources from those carrying noise for the task at hand have become critical issues.
Some attempts have been made in this direction.
For example,
~\cite{lim2012transfer,vezhnevets2014associative}
used taxonomies to leverage learning from few examples on the SUN09 dataset.
In~\cite{lim2012transfer}, authors attacked the transfer learning problem on the SUN09 dataset by using additional data from another dataset.
Zero-shot  approaches
were investigated by~\cite{rohrbach2011evaluating} on a subset of the Imagenet dataset.
Large-scale visual detection has been explored by~\cite{vezhnevets2014associative}.
However, all these approaches
assume access to all source training data. 
%
A slightly different approach to transfer learning that aimed to cirumvent this limitation, is reuse of a large convolutional neural network pre-trained on a large visual recognition dataset.
The simplest approach is to use outputs of intermediate layers of such a network, such as DeCAF~\cite{decaf} or Caffe~\cite{caffe}.
A more sophisticated way of reuse is fine-tuning, a kind of warm-start, that has been successfully exploited in visual detection~\cite{rcnn_detection} and domain adaptation~\cite{ganin2015unsupervised,long2015learning}.

In many of these works the use of richer sources of information has been supported by an increase in the information available in the target domain as well. From an intuitive point of view, this corresponds to having more data points than dimensions. Of course, this makes the learning and selection process easier, but in many applications it is not a reasonable hypothesis.
Also, none of the proposed algorithms has a theoretical backing.

%
%
While not explicitly mentioned before, the problem outlined above can also be viewed as a learning scenario where the number of features is by far larger than the number of training examples.
Indeed, learning with classeme-like features~\cite{classemes,li2010object} when only few training examples are available can be seen as a \acl{HTL} problem.
Clearly, a pure empirical risk minimization would fail due to severe overfitting.
In machine learning and statistics this is known as a feature selection problem, and is usually addressed by constraining or penalizing the solution with sparsity-inducing norms.
One important sparsity constraint is a non-convex $L0$ pseudo-norm constraint $\|\bw\|_0 \leq k$, that simply corresponds to choosing up to $k$ non-zero components of a vector $\bw$.
One usually resorts to the \emph{subset selection} methods, and greedy algorithms for obtaining solutions under this constraint~\cite{das2008algorithms,das2011submodular,zhang2008adaptive,zhang2009consistency}.
However, in some problems introducing $L0$ constraint might be computationally difficult.
There, a computationally easier alternative is a convex relaxation of $L0$, the $L1$ regularization.
Empirical error minimization with $L1$ penalty with various loss functions (for square loss is known as Lasso)
has many favorable properties and
is well studied theoretically~\cite{buhlmann2011statistics}.
Yet, $L1$ penalty is known to suffer from several limitations, one of which is poor empirical performance when there are many correlated features.
Perhaps the most famous way to resolve this issue is an \emph{elastic net} regularization which is a weighted mixture of $L1$ and squared $L2$ penalties~\cite{hastie2009elements}.
Since our work partially falls into the category of feature selection, we have extensively evaluated the aforementioned baselines in our task.
As it will be shown below, none of them achieves competitive performances compared to our approach.

\section{Transfer Learning through Subset Selection}
\label{tl}
\noindent\textbf{\emph{Definitions.}}
%
%
We will denote with small and capital bold letters respectively
column vectors and matrices, e.g. $\ba=[a_1, a_2, \ldots, a_d]^T\in \R^d~$
and $\bA \in \R^{d_1 \times d_2 }~$.
The subvector of $\ba$ with rows indexed by set $S$ is $\ba_S$, while
the square submatrix of $\bA$ with rows and columns indexed by set $S$ is $\bA_S$.
For $\bx \in \R^d$, the \emph{support} of $\bx$ is $\supp(\bx) = \{i \in \{1, \ldots, d\} \colon x_i \neq 0\}$.
%
Denoting by $\sX$ and $\sY$ respectively the input and output space of the learning problem,
the training set is $\{(\bx_i,y_i)\}_{i=1}^m$, drawn i.i.d. from the probability distribution $p$ defined over $\sX \times \sY$.
We will focus on the binary classification problem so $\sY = \{-1, 1\}$, and, without loss of generality, $\sX = \{\bx : \|\bx\|_2 \leq 1, \bx \in \reals^d\}$.


To measure the accuracy of a learning algorithm, we have a non-negative \emph{loss} function
$\ell(h(\bx), y)$, which measures the cost incurred predicting $h(\bx)$ instead of $y$.
In particular, we will focus on the square loss, $\ell(h(\bx), y)=(h(\bx)-y)^2$, for its appealing computational properties.
The \emph{risk} of a hypothesis $h$, with respect to the probability distribution $p$, is then defined as
$
\Risk{}(h) := \E_{(\bx,y) \sim p}[\ell(h(\bx), y)]
$,
while the \emph{empirical risk} given a training set $\{(\bx_i,y_i)\}_{i=1}^m$ is
$
\Riskh(h) := \frac{1}{m} \sum_{i=1}^m \ell(h(\bx_i), y_i)
$.
Whenever the hypothesis is a linear predictor, that is, $h_{\bw}(\bx) = \bw\tp \bx$, we will also use risk notation as $\Risk{}(\bw) = \Risk{}(h_{\bw})$ and $\Riskh(\bw) = \Riskh(h_{\bw})$.
~\\\noindent\textbf{\emph{Source Selection.}}
%
Assume, that we are given a finite source hypothesis set $\{\hsrc_i\}_{i=1}^n$ and the training set $\{(\bx_i,y_i)\}_{i=1}^m$.
As in previous works~\cite{mansour2009domain,tommasi2013learning,jie2011multiclass}, we consider the target hypothesis to be of the form
\begin{equation}
\label{eq:eq_transf}
\htrg_{\bw, \bbeta}(\bx) = \bw\tp \bx + \sum_{i=1}^n \beta_i \hsrc_i(\bx),
\end{equation}
where $\bw$ and $\bbeta$ are found by the learning procedure.
The essential parameter here is $\bbeta$, that is the one controlling the influence of each source hypothesis.
Previous works in transfer learning have focused on finding $\bbeta$ such that it minimizes the error on the training set, subject to some condition on $\bbeta$.
In particular,~\cite{tommasi2013learning}  proposed to minimize the leave-one-out error w.r.t. $\bbeta$, subject to $\|\bbeta\|_2 \leq \tau$, which is known to improve generalization for the right choice of $\tau$~\cite{kuzborskij2013stability}.
A slightly different approach is to use $\|\bbeta\|_1 \leq \tau$ regularization for this purpose~\cite{tommasi2013learning}, that induces solutions with most of the coefficients equal to $0$, thus assuming  that the optimal $\bbeta$ is sparse.

In this work we embrace a weaker assumption, namely, there exist up to $k$ sources that collectively improve the generalization on the target domain.
Thus, we pose the problem of the Source Selection as a minimization of the regularized empirical risk on the target training set, 
while constraining the number of selected source hypotheses.
\begin{nameddef}[$k$-Source Selection]
\label{def:htl_subset_selection}
Given the training set
$\left\{\left([\bx_i\tp, \hsrc_1(\bx_i), \ldots, \hsrc_n(\bx_i)]\tp, y_i\right)\right\}_{i=1}^m$
we have the optimal target hypothesis $\htrg_{\bw^\star, \bbeta^\star}$ by solving,
\begin{align}
(\bw^\star, \bbeta^\star) &= \argmin_{\bw, \bbeta}~\left\{ \Riskh(\htrg_{\bw, \bbeta}) + \lambda \|\bw\|_2^2 + \lambda \|\bbeta\|_2^2\right\}, \nonumber \\
&\textnormal{ s.t }~ \|\bw\|_0 + \|\bbeta\|_0 \leq k. \label{eq:htl_subset_selection}
\end{align}
\end{nameddef}
Notably, the problem~\eqref{eq:htl_subset_selection} is a special case of the \emph{Subset Selection} problem~\cite{das2008algorithms}: choose a subset of size $k$ from the $n$ observation variables, which collectively give the best prediction on the variable of interest. 
%
However, the Subset Selection problem is \textbf{NP}-hard~\cite{das2008algorithms}.
In practice we can resort to algorithms generating approximate solutions, for many of which we have approximation guarantees.
Hence, due to the extensive practical and theoretical results, we will treat the $k$-Source Selection as a Subset Selection problem,
building atop of existing guarantees.
%

We note that our formulation,~\eqref{eq:htl_subset_selection}, differs from the classical subset selection for the fact that it is $L2$-regularized.
This technical modification makes an essential practical and theoretical difference and it is the crucial part of our algorithm.
First, $L2$ regularization is known to improve the generalization ability of empirical risk minimization. 
Second, we show that regularization also improves the quality of the approximate solution in situations when the sources, or features, are correlated.
At the same time, the experimental evaluation corroborates our theoretical findings:
Our formulation substantially outperforms standard subset selection, feature selection algorithms, and transfer learning baselines.
%
\section{Greedy Algorithm for $k$-Source Selection}
\label{greedy-tl}
In this section we state the algorithm proposed in this work, \verb!GreedyTL!\footnote{Source code is available at~\url{http://idiap.ch/~ikuzbor/}}.
In the following we will denote by $U=\{1, \ldots, n+d\}$ the index set of all available source hypotheses and features, and by $S$, the index set of selected ones.
%
\begin{nameddef}[GreedyTL]
Let $\bX \in \reals^{m \times d}$ and $\by \in \{+1,-1\}^m$ be the zero-mean unit-variance training set, $\{\hsrc_i\}_{i=1}^n$, source hypothesis set, and
$k$ and $\lambda$, regularization parameters.
Then, denote $\bC = \bZ\tp \bZ$ and $\bb = \bZ\tp \by$, where
  $
  \bZ = \left[\bX {\tiny \begin{array}{ccc} \hsrc_1(\bx_1) & \cdots & \hsrc_n(\bx_1) \\ \cdots & \cdots & \cdots \\ \hsrc_1(\bx_m) & \cdots & \hsrc_n(\bx_m) \end{array}}\right]~,
  $
and
select set $S$ of size $k$ as follows: (I) Initialize $S \gets \varnothing$ and $U \gets \{1, \ldots, n+d\}$. (II) Keep populating $S$ with $i \in U$, that maximize $\bb_S\tp ((\bC + \lambda \bI)_S^{-1})\tp \bb_S$, as long as $|S| \leq k$ and $U$ is non-empty.
\end{nameddef}

In this basic formulation, the algorithm requires to invert a $(d+n)$-by-$(d+n)$ matrix at each iteration of a greedy search.
Clearly, this naive approach gets prohibitive with the growth of the number of source hypotheses, feature dimensions, and desired subset size, since its computational complexity would be in $\scO(k (d+n)^4)$.
However, we note that in transfer learning one typically assumes that training set is much smaller than number of sources and feature dimension.
For this reason we apply rank-one updates w.r.t. the dual solution of regularized subset selection, so that the size of the inverted matrix does not change.
The computational complexity then improves to $\scO(k (d+n) m^2)$. 
We present the pseudocode of such a variant of our algorithm, \textbf{GreedyTL with Rank-One Updates} in Algorithm~\ref{alg:greedytl_rankone}. The computational complexity of the operations is shown at the end of each line.
\renewcommand{\algorithmicrequire}{\textbf{Input:}}
\renewcommand{\algorithmicensure}{\textbf{Output:}}
\begin{algorithm}
\caption{GreedyTL with Rank-One Updates \label{alg:greedytl_rankone}}
\begin{algorithmic}[1]
\Require{
$\bZ \in \reals^{m \times (d+n)}$ -- $m$ examples formed from features and source predictions,\\
$\by \in \spin^m$ -- labels,\\
$k \in \{1,\ldots,d+n\}, \lambda \in \reals_+$ -- hyperparameters.
}
\Ensure{$\bw$ -- target predictor.}
\State $U \gets \{1, \ldots, d + n\}$ \Comment{All candidates}
\State $S \gets \varnothing$ \Comment{Selected sources and features}
\State $\bK \gets [\bzero, \ldots, \bzero] \in \reals^{m \times m}$
\State $\bG \gets \lambda^{-1}\bI \in \reals^{m \times m}$
\While {$U \neq \varnothing$ \textbf{and} $|S| \leq k$}
\State
\[i^\star \gets \argmax_{i \in U} \left\{ \by\tp (\bK + \bz_i \bz_i\tp) \bG' \by \ \middle| \
\bG' \gets \bG - \frac{\bG \bz_i \bz_i\tp \bG}{1 + \bz_i\tp \bG \bz_i} \right\}
\] \Comment{$\scO((d+n)(m^2 + m))$}\\
\State \quad\quad Computing $\bG'$: \Comment{$\scO(m^2 + m)$}
\State \quad\quad Computing score of $i$: \Comment{$\scO(m^2 + m)$}
%
\State $S \gets S \cup \{i^\star\}$
\State $U \gets U \setminus \{i^\star\}$
\State $\bK \gets \bK + \bz_{i^\star} \bz_{i^\star}\tp$ \Comment{$\scO(m^2)$}
\State $\bG \gets \bG - \frac{\bG \bz_{i^\star} \bz_{i^\star}\tp \bG}{1 + \bz_{i^\star}\tp \bG \bz_{i^\star}}$ \Comment{$\scO(m^2 + m)$}
\State 
\EndWhile \Comment{$\scO(k (d+n) m^2)$}
\State $\bw \gets \bzero \in \reals^{d+n}$
\State $w_i \gets \bz_i\tp \bG \by, \ \forall i \in S$
\end{algorithmic}
\end{algorithm}

%
~\\\noindent\textbf{\emph{Derivation of the Algorithm.}}
\label{sec:derivation}
We derive \verb!GreedyTL! by extending the well known \ac{FR} algorithm~\cite{das2008algorithms}, which gives an approximation to the subset selection problem, the problem of our interest.
\ac{FR} is known to find a good approximation as far as features are uncorrelated~\cite{das2008algorithms}.
In the following, we build upon \ac{FR} by introducing a Tikhonov ($L2$) regularization into the formulation.
The purpose of regularization is twofold: first, it improves the generalization ability of the empirical risk minimization, and second,
it makes the algorithm more robust to the feature correlations, thus opting to find better approximate solution.

First, we briefly formalize the subset selection problem. In a subset selection problem one tries to achieve a good prediction accuracy on the \emph{predictor} random variable $Y$,
given a linear combination of a subset of the \emph{observation} random variables $\{X_i\}_{i=1}^n$.
The least squares subset selection then reads as
\[
\min_{|S|= k, \bw \in \reals^k} \E \left[ \left(Y - \sum_{i \in S} w_i X_i \right)^2\right].
\]
%
Now denote the covariance matrix of zero-mean unit-variance observation random variables by $\bC$ (a correlation matrix),
%
and the correlations between $Y$ and $\{X_i\}_{i=1}^m$ as $\bb$.
Note that the zero-mean unit-variance assumption will be necessary to prove the theoretical guarantees of our algorithm.
By virtue of the analytic solution to least-squares and using the introduced notation, we can also state the
equivalent \emph{Subset Selection problem}:
$\max_{|S|=k} \bb_S\tp (\bC_S^{-1})\tp \bb_S~$.
However, our goal is to obtain the solution to~\eqref{eq:htl_subset_selection}, or a \emph{$L2$-regularized} subset selection.
Similarly to the unregularized subset selection, it is easy to get that~\eqref{eq:htl_subset_selection} is equivalent to
$
\max_{|S|=k} \bb_S\tp ((\bC_S + \lambda \bI)^{-1})\tp \bb_S.
$
%
%
%
%
%
As said above, the Subset Selection problem is \textbf{NP}-hard, however, there are several ways to approximate it in practice~\cite{das2011submodular}.
We choose \ac{FR} for this task for its simplicity, appealing computational properties and provably good approximation guarantees.
%
%
Now, to apply \ac{FR} to our problem, all we have to do is to provide it with normalized matrix $(\bC + \lambda \bI)^{-1}$ instead of $\bC^{-1}$.
~\\\\
\textbf{\emph{Approximated Randomized Greedy Algorithm.}}
As mentioned above, the complexity of GreedyTL is linear in $d+n$, the number of features and the size of the source hypothesis set. In particular, the search in $U$ for the index to add to $S$ is responsible for the dependency on $d+n$. Here we show how to approximate this search with a randomized strategy.
We will use the following Theorem.

\begin{theorem}[\cite{SmolaS02}(Theorem 6.33)]
Denote by $M := \{x_1, \ldots , x_m\} \subset \R$ a set of cardinality $m$, and by $\tilde{M} \subset M$ a random subset of size $\tilde{m}$. Then the probability that $\max \tilde{M}$ is greater or equal than $n$ elements of $M$ is at least 
$1 - ( \frac{n}{m} ) ^{\tilde{m}} $.
\end{theorem}

The surprising consequence is that, in order to approximate the maximum over a set, we can use a random subset of size $\scO(1)$. In particular, if we want to obtain results in the $\frac{n}{m}$ percentile range with $1 -\eta$ confidence, we use\footnote{Note that the formula for $\tilde{m}$ in \cite{SmolaS02} contains an error, the correct one is the one we report.}  $\tilde{m} =\frac{\log(\eta)}{\log \frac{n}{m}}$. Practically, if we desire values that are better than $95\%$ of all other estimates with $1-0.05$ probability, then $59$ samples are sufficient. This rule is commonly called the 59-trick and it has been widely used to speed-up a wide range of algorithms with negligible loss of accuracy, e.g.~\cite{DomingoW00,SmolaS00}. Indeed, as we will show in Section~\ref{sec:exps_approximated_greedytl}, we virtually don't lose any accuracy using this strategy.

With the 59-trick, the search in $U$ becomes a search for the maximum over a random set of size 59. So, the overall complexity is reduced to $\scO(k m^2)$, that is \emph{independent} from all the quantities that are expected to be big.

~\\\\
\textbf{\emph{Theoretical Guarantees.}}
%
We now focus on the analysis of the generalization properties of \verb!GreedyTL! for solving $k$-Source Selection problem~\eqref{eq:htl_subset_selection}.
Throughout this paragraph we will consider a truncated target predictor $\htrg_{\bw, \bbeta}(\bx) := \scT\left(\bw\tp \bx + \sum_{i=1}^n \beta_i \hsrc_i(\bx)\right)$, with $\scT(a) := \min\{\max\{a, -1\}, 1\}$.
We will also use big-O notation $\tilde{\scO}$ to indicate the supression of a logarithmic factor, in other words, $f(x) \in \tilde{\scO}(g(x))$ is a short notation for $\exists n ~:~ f(x) \in \tilde{\scO}(g(x) \log^n g(n))$.
First we state the bound on the risk of an approximate solution returned by \verb!GreedyTL!.
\footnote{Proofs for theorems can be found in the appendix.}

\begin{theorem}
\label{thm:gen_no_approx}
Let \verb!GreedyTL! generate the solution $(\bhatw, \bhatbeta)$, given the training set $(\bX, \by)$, source hypotheses $\{\hsrc_i\}_{i=1}^n$ with $\src{\tau_\infty} := \max_i\{\|\hsrc_i\|^2_\infty\}$, hyperparameters $\lambda$ and $k$.
Then with high probability,
{\small
\begin{equation*}
  \Risk\left(\htrg_{\bhatw, \bhatbeta}\right) - \Riskh\left(\htrg_{\bhatw, \bhatbeta}\right) \leq  \tilde{\scO}\left( \frac{ 1 + k \src{\tau_\infty} }{\lambda m}
  + \sqrt{\src{\Riskh} \frac{1 + k \src{\tau_\infty}}{\lambda m} } \right),
\end{equation*}
}
where
{\small
$
\src{\Riskh} := \frac{1}{m}\sum_{i=1}^m \ell\left(y_i, \scT\left(\sum_{j \in \supp(\bhatbeta)} \hat{\beta}_i \hsrc_j(\bx_i)\right)\right).
$
}
\end{theorem}
This results in a generalization bound which tells us how close the performance of the algorithm on the test set will be to the one on the training set.
The key quantity here is $\src{\Riskh}$, which captures the quality of the sources selected by the algorithm.
To understand its impact, assume that $\lambda = \scO(1)$.
The bound has two terms, a fast one of the order of $\tilde{\scO}\left(k/m\right)$ and a slow one of the order $\tilde{\scO}\left(\sqrt{\src{\Riskh} k/m}\right)$. When $m$ goes to infinity and $\src{\Riskh}\neq0$ the slow term will dominate the convergence rate, giving us a rate of the order of $\tilde{\scO}\left(\sqrt{\src{\Riskh} k/m } \right)$.
If $\src{\Riskh}=0$ the slow term completely disappears, giving us a so called fast rate of convergence of $\tilde{\scO}(k/m)$.
On the other hand, for any finite $m$ of the order of $\tilde{\scO}(k/\src{\Riskh})$, we still have a rate of the order of $\tilde{\scO}(k/m)$.
Hence, the quantity $\src{\Riskh}$ will govern the finite sample and asymptotic behavior of the algorithm, predicting a faster convergence in both regimes when it is small.
In other words, when the source and target tasks are similar, TL facilitates a faster convergence of the empirical risk to the risk.
A similar behavior was already observed in~\cite{kuzborskij2013stability,Ben-DavidU13}.

However, one might ask what happens when the selected sources are providing bad predictions.
Since $\src{\Riskh} \leq 1$, due to truncation, the empirical risk converges to the risk at the standard rate $\tilde{\scO}(\sqrt{k/m})$, the same one we would have without any transfering from the sources classifiers.

We now present another result that upper bounds the difference between the risk of solution of the algorithm and the empirical risk of the optimal solution to the $k$-Source Selection problem.
\begin{theorem}
\label{thm:gen_approx}
In addition to conditions of Theorem~\ref{thm:gen_no_approx},
let $(\bw^\star, \bbeta^\star)$ be the optimal solution to~\eqref{eq:htl_subset_selection}.
Given a sample correlation matrix $\bhatC$, assume that $\hat{C}_{i,j\neq i} \leq \gamma < \frac{1 + \lambda}{6 k}$, and $\epsilon:=\frac{16(k + 1)^2 \gamma}{1 + \lambda}$.
Then with high probability,
%
\begin{align*}
\small
  \Risk\left(\htrg_{\bhatw, \bhatbeta}\right) - \Riskh\left(\htrg_{\bw^\star, \bbeta^\star}\right) \leq (1+\epsilon)\src{\Riskh_\lambda}
+ \tilde{\scO}\left( \frac{1 + k \src{\tau_\infty} }{\lambda m} + \sqrt{\src{\Riskh_\lambda} \frac{1 + k \src{\tau_\infty}}{\lambda m} } \right),
\end{align*}
where
{\small$
  \src{\Riskh_\lambda} := \min_{|S| \leq k}\left\{ \frac{\lambda}{|S|} + \frac{1}{|S|} \sum_{i \in S} \Riskh(\hsrc_i) \right\}.
$}
\end{theorem}
To analyze the implications of Theorem~\ref{thm:gen_approx}, let us consider few interesting cases.
Similarly as done before, the quantity $\src{\Riskh_\lambda}$ captures how well the source hypotheses are aligned with the target task and governs the asymptotic and finite sample regime.
In fact,
assume for any finite $m$ that there is at least one source hypothesis with small empirical risk, in particular, in
$\tilde{\scO}(\sqrt{k / m})$,
and set $\lambda = \tilde{\scO}(\sqrt{k / m})$.
Then we have that
$
\Risk(\htrg_{\bhatw, \bhatbeta}) - \Riskh(\htrg_{\bw^\star, \bbeta^\star}) = \tilde{\scO}\left(\sqrt{k / m}\right),
$
that is we get the generalization bound as if we are able to solve the original \textbf{NP}-hard problem in \eqref{eq:htl_subset_selection}.
In other words, if there are useful source hypotheses, we expect our algorithm to perform similarly to the one that identifies the optimal subset.
This might seem surprising, but it is important to note that we do not actually care about identifying the correct subset of source hypotheses.
We only care about how well the returned solution is able to generalize.
%
On the other hand, if not even one source hypothesis has low risk, selecting the best subset of $k$ sources becomes meaningless. In this scenario, we expect the selection of any subset to perform
in the same way. Thus the approximation guarantee does not matter anymore.
%

We now state the approximation guarantees of \verb!GreedyTL! used to prove Theorem~\ref{thm:gen_approx}.
In the following Corollary we show how far the optimal solution to the regularized subset selection is from the approximate one found by \verb!GreedyTL!.
\begin{cor}
  \label{cor:tikhonov_fr_approx}
Let $\lambda \in \reals^+$ and $k \leq n$.
Denote $\OPT := \min_{\|\bw\|_0 = k} \left\{ \Riskh(\bw) + \lambda \|\bw\|_2^2 \right\}$.
Assume that $\bhatC$ and $\bhatb$ are normalized, and $\hat{C}_{i,j\neq i} \leq \gamma < \frac{1 + \lambda}{6 k}$.
Then, \ac{FR} algorithm generates an approximate solution $\bhatw$ to the regularized subset selection problem that satisfies
$
  \Riskh(\bhatw) + \lambda \|\bhatw\|_2^2 \leq \left(1 + \frac{16(k + 1)^2 \gamma}{1 + \lambda} \right) \OPT - \frac{16 (k+1)^2 \gamma \lambda}{(1+\lambda)^2}.
$
\end{cor}
Apart from being instrumental in the proof of Theorem~\ref{thm:gen_approx}, this statement also points to the secondary role of the regularization parameter $\lambda$: unlike in \ac{FR}, we can control the quality of the approximate solution even if the features are correlated.

\section{Experiments}
\label{sec:exps}
In this section we present 
experiments
comparing \verb!GreedyTL! to several 
transfer learning and 
feature selection algorithms.
As done previously, we considered the 
object detection task and,
for all datasets,
 we left out one class considering it as the target class, while the remaining classes were treated as sources~\cite{tommasi2013learning}.
We repeated this procedure for every class and for every dataset at hand, and averaged the performance scores.
In the following, we refer to this procedure as \emph{leave-one-class-out}.
We performed the evaluation for every class, 
reporting averaged class-balanced recognition scores.


We used subsets of Caltech-256~\cite{griffinHolubPerona}, Imagenet~\cite{deng2009imagenet}, SUN09~\cite{choi2010exploiting}, SUN-$397$~\cite{xiao2010sun}.
The largest setting considered involves $1000$ classes, totaling in $1.2$M examples, where the number of training examples of the target domain varies from $11$ to $20$.
Our experiments aimed at verifying three claims:
\begin{enumerate}[I.]
\item $L2$-regularization is important when using greedy feature selection as a transfer learning scheme.
\item In a small-sample regime \verb!GreedyTL! is more robust than alternative feature selection approaches, such as $L1$-regularization.
\item The approximated randomized greedy algorithm improves the computational complexity of \verb!GreedyTL! with no significant loss in performance.
\end{enumerate}
%
%
\subsection{Datasets and Features}
We used
the whole Caltech-256, 
a public subset of Imagenet containing $10^3$ classes, 
all the classes of SUN09 that have more than $1$ example, which amounts to $819$ classes,
and the whole SUN-$397$ dataset containing $397$ place categories.
For Caltech-256 and Imagenet, we used as features the publicly-available $1000$-dimensional SIFT-BOW 
descriptors, while for SUN09 we extracted $3400$-dimensional PHOG descriptors.
In addition, for Imagenet and SUN-$397$, we also ran experiments using convolutional features extracted from DeCAF neural network~\cite{decaf}.

We composed a negative class by merging $100$ held-out classes (\emph{surrogate} negative class).
We did so for each dataset, and we further split it into the \emph{source} negative and the \emph{target} negative class as $90\%+10\%$ respectively, for training sources and the target.
The source classifiers were trained for each class in the dataset, combining all the positive examples of that class and the source negatives.
On average, each source classifier was trained using $10^4$ examples for the Caltech-256, $10^5$ for Imagenet and $10^3$ for the SUN09 dataset.
%
The training sets for the target task were composed by $\{2, 5, 10\}$ positive examples, and $10$ negative ones.
Following~\cite{tommasi2013learning}, the testing set contained $50$ positive and $50$ negative examples for Caltech-256, Imagenet, and SUN-$397$.
For the skewed SUN09 dataset we
took one positive and $10$ negative training examples, with the rest left for testing.
We drew each target training and testing set randomly $10$ times, averaging the results over them.
%
%
\subsection{Baselines}
\label{sec:baselines}
%
We chose a linear SVM to train the source classifiers~\cite{fan2008liblinear}.
This allows us to compare fairly with relevant baselines (like Lasso) and is in line with recent trends in large scale visual recognition and transfer learning \cite{decaf}.
The models were selected
by $5$-fold cross-validation having regularization parameter $C \in \{10^{-4}, 10^{-3}, \cdots, 10^{4}\}$.
In addition to trained source classifiers, for the Caltech-256, we also evaluated transfer from Classemes~\cite{classemes} and Object Bank~\cite{li2010object}, which are very similar in spirit to source classifiers.
At the same time, for Imagenet, we evaluated transfer from the outputs of the final layers of the DeCAF convolutional neural network~\cite{decaf}.

We divided the baselines into two groups - the linear transfer learning baselines that do not require access to the source data, and the feature selection baselines.
%
We included the second group of baselines due to \verb!GreedyTL!'s resemblance to a feature selection algorithm.
We focus on the linear baselines, since we are essentially interested in the feature selection in high-dimensional spaces from few examples.
In that scope, most feature selection algorithms, such as Lasso, are linear.
In particular, amongst TL baselines we chose:
\emph{No transfer}: \ac{RLS} algorithm trained solely on the target data;
\emph{Best source}: indicates the performance of the best source classifier selected by its score on the testing set. This is a pseudo-indicator of what an \ac{HTL} can achieve;
\emph{AverageKT}: obtained by averaging the predictions of all the source classifiers;
\emph{RLS src+feat}: \ac{RLS} trained on the concatenation of feature descriptors and source classifier predictions;
\emph{MultiKT $\|\cdot\|_2$}: \ac{HTL} algorithm by~\cite{tommasi2013learning} selecting $\bbeta$ in~\eqref{eq:eq_transf} by minimizing the leave-one-out error subject to $\|\bbeta\|_2 \leq \tau$;
\emph{MultiKT $\|\cdot\|_1$}: similar to previous, but applying the constraint $\|\bbeta\|_1 \leq \tau$;
\emph{DAM}: An \ac{HTL} algorithm by~\cite{duan2009domain}, that can handle selection from multiple source hypotheses. It was shown to perform better than the well known and similar ASVM~\cite{yang2007cross} algorithm.
For the feature selection baselines we selected well-established algorithms involving sparsity assumption:
\emph{L1-Logistic}: Logistic regression with $L1$ penalty~\cite{hastie2009elements};
\emph{Elastic-Net}: Logistic regression with mixture of $L1$ and $L2$ penalties~\cite{hastie2009elements};
\emph{Forward-Reg}: Forward regression -- a classical greedy feature selection algorithm.
When comparing our algorithms to the baselines on large datasets, we also consider a Domain Adaptive \emph{Dictionary Learning} baseline~\cite{qiu2012domain}.
This baseline represents the family of dictionary learning methods for domain adaptation and transfer learning.
In particular, it learns a dictionary on the source domain and adapts it to the target one.
However, in our setup the only access to the source data is through the source hypotheses. Therefore, the only way to construct source features is by using the source hypotheses on the target data points.
%
\newlength{\subfiglength}
%
%
%
\subsection{Results}
Figure~\ref{fig:accuracies} shows the leave-one-class-out performance. 
In addition, Figures~\ref{fig:acc_caltech256_classemes}, \ref{fig:acc_caltech256_object_bank}, \ref{fig:acc_imagenet_decaf} show the performance when transferring from off-the-shelf classemes, object-bank feature descriptors, and DeCAF neural network activations.
\begin{figure*}[!t]
  \caption{Performance on the Caltech-256, subsets of Imagenet (1000 classes) and SUN09 (819 classes).
  Averaged class-balanced accuracies in the leave-one-class-out setting.
  }
  \centering
  \setlength{\subfiglength}{4.8cm}
  \subfloat[Caltech-256]{%
    \includegraphics[height=\subfiglength]{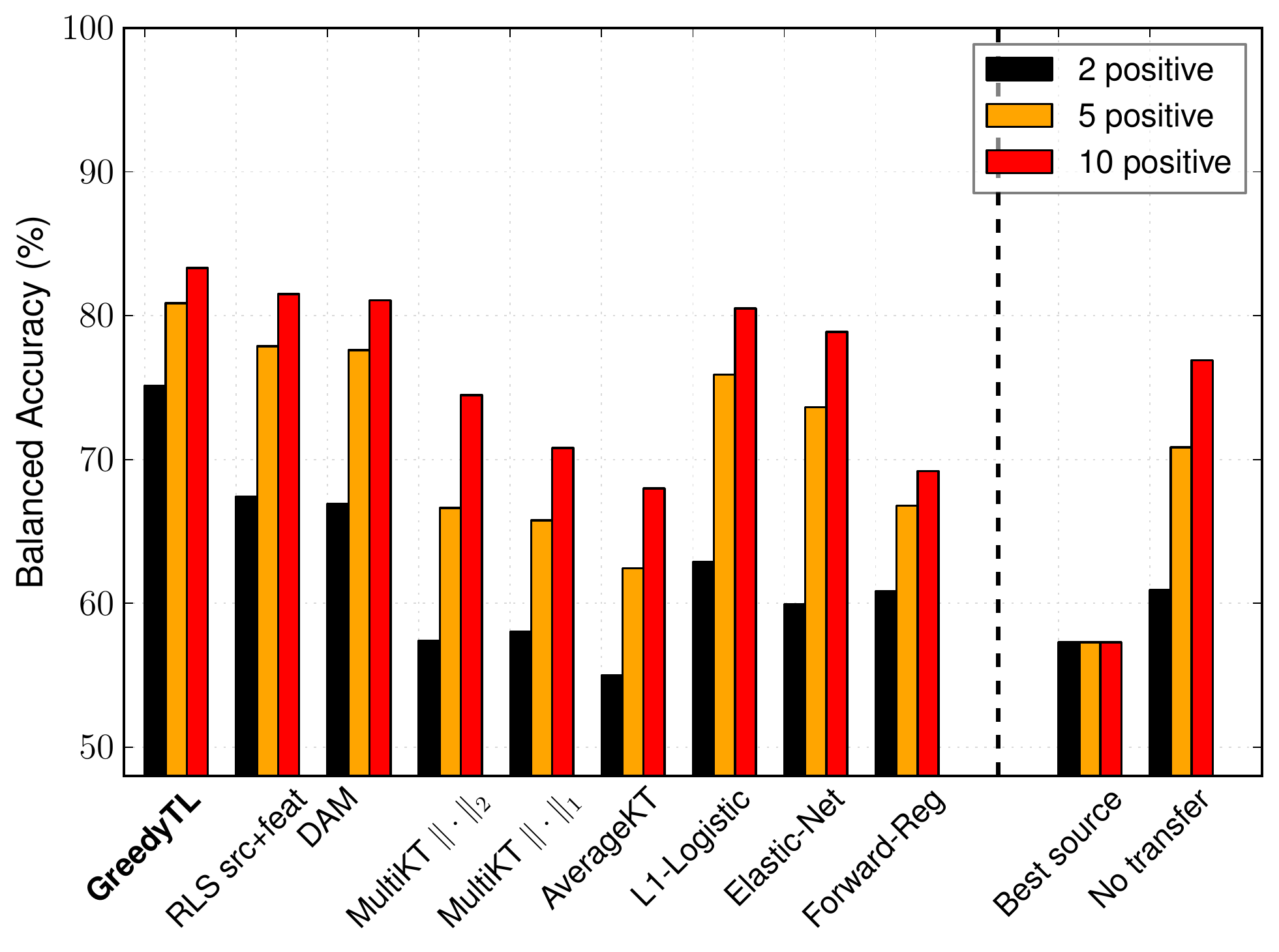}
    \label{fig:acc_caltech256}
  }%
  \subfloat[Caltech-256 (Classemes)]{%
    \includegraphics[height=\subfiglength]{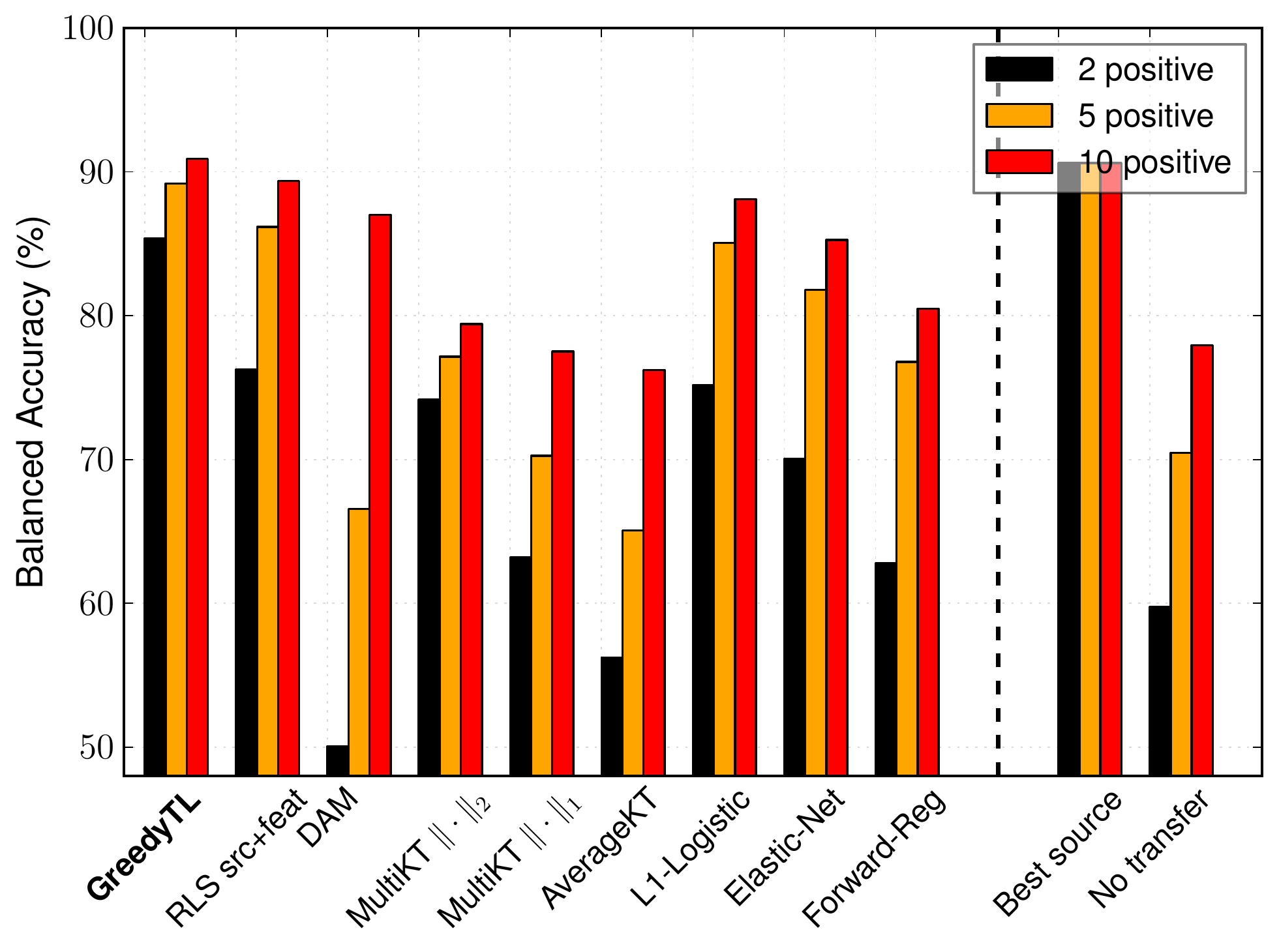}
    \label{fig:acc_caltech256_classemes}
  }\\
  \vspace{-0.2cm}
  \subfloat[Caltech256 (Object Bank)] {%
    \includegraphics[height=\subfiglength]{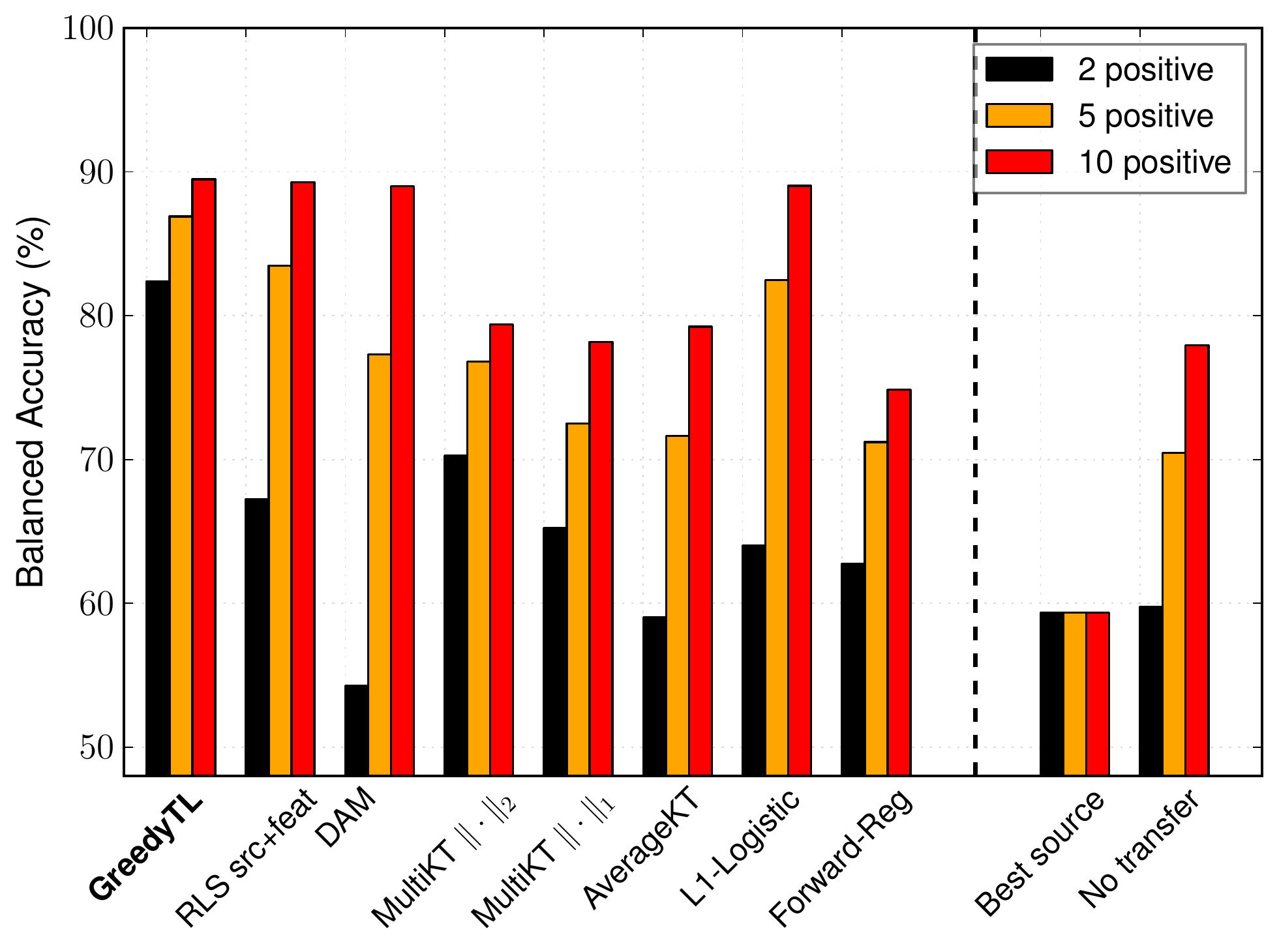}
    \label{fig:acc_caltech256_object_bank}
  }%
  \subfloat[SUN09] {%
    \includegraphics[height=\subfiglength]{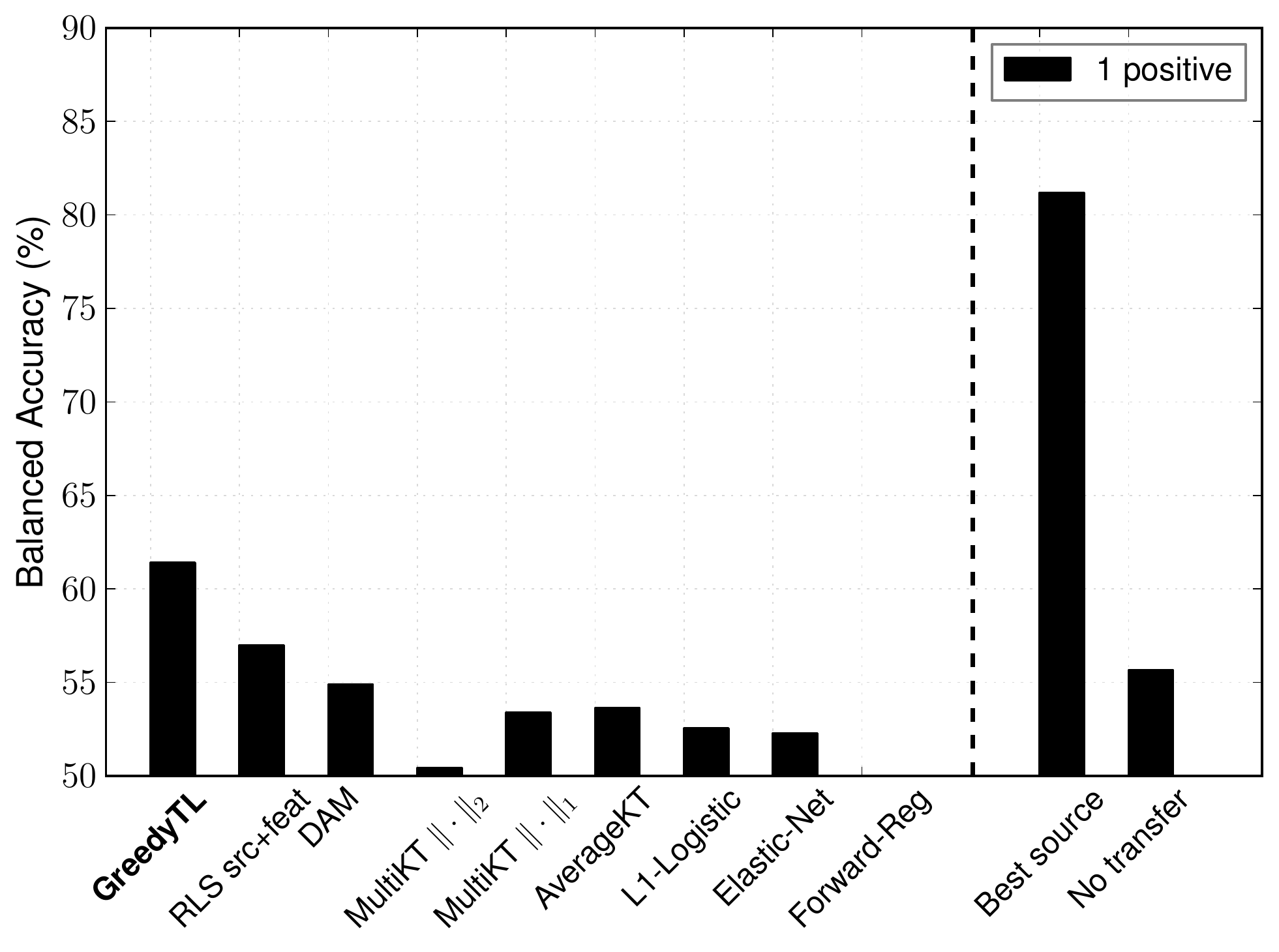}
    \label{fig:acc_sun09}
  }\\
  \vspace{-0.2cm}
  \subfloat[Imagenet ($1000$ classes)] {%
    \includegraphics[height=\subfiglength]{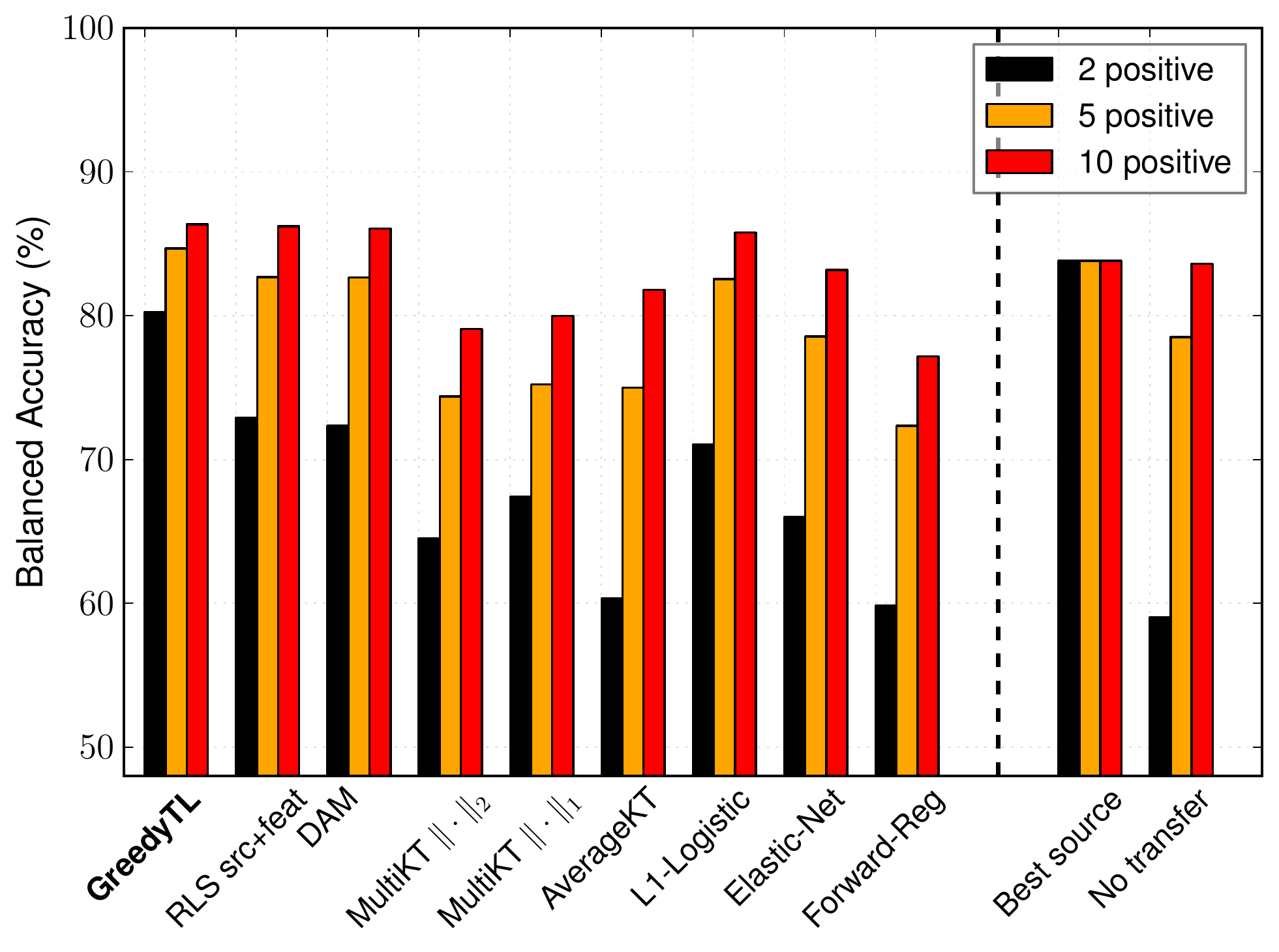}
    \label{fig:acc_imagenet}
  }%
  \subfloat[Imagenet (sources are DeCAF outputs, $1000$ classes)] {%
    \includegraphics[height=\subfiglength]{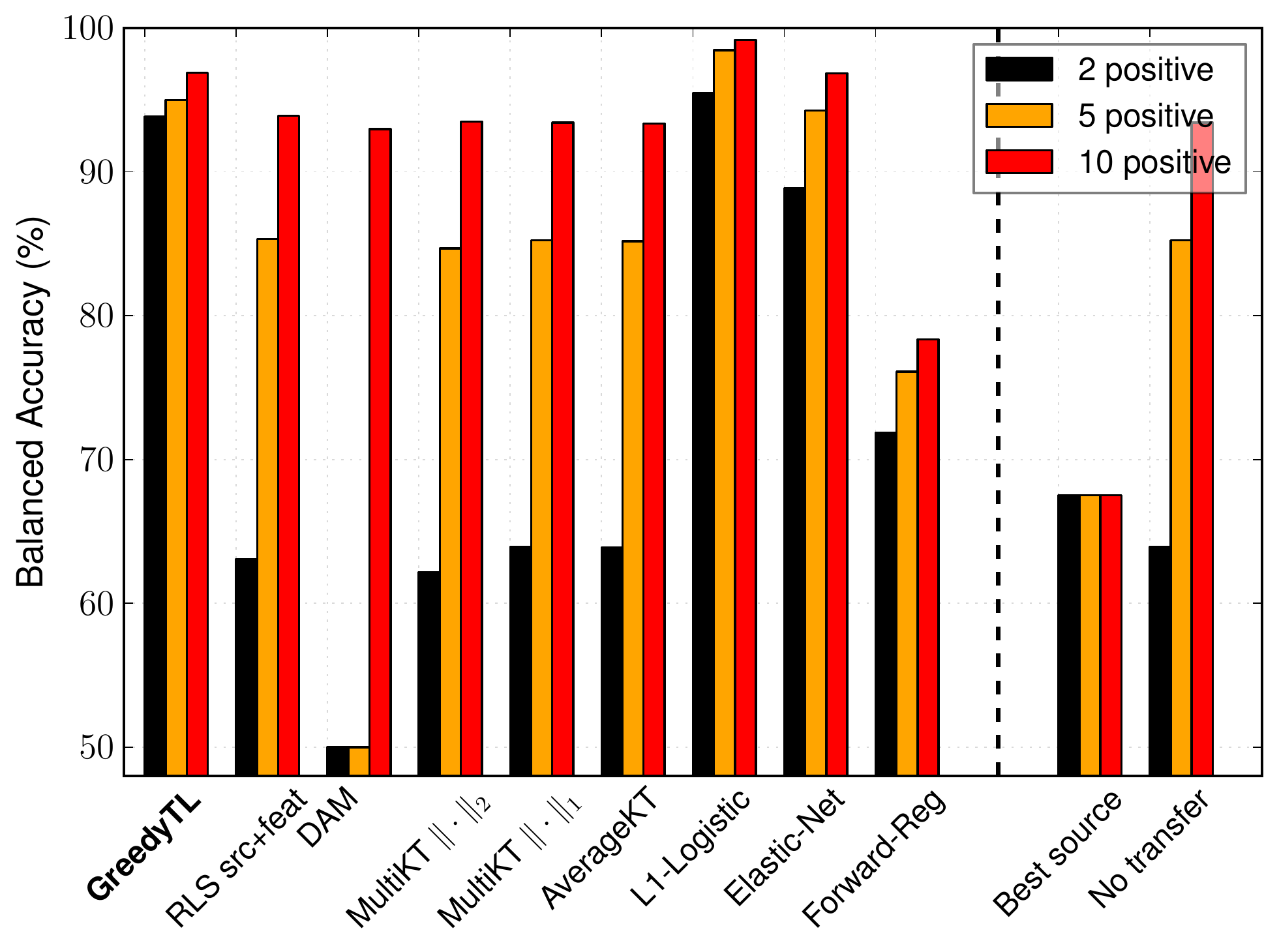}
    \label{fig:acc_imagenet_decaf}
  }\\
\label{fig:accuracies}
\end{figure*}
Whenever any baseline algorithm has hyperparameters to tune, we chose the ones that minimize the leave-one-out error on the training set.
In particular, we selected the regularization parameter $\lambda \in \{10^{-4}, 10^{-3}, \ldots, 10^4\}$.
MultiKT and DAM have an additional hyperparameter that we call $\tau$ with $\tau \in \{10^{-3}, \ldots, 10^3\}$.
Kernelized algorithms were supplied with a linear kernel.
%
Model selection for \verb!GreedyTL! involves two hyperparameters, that is $k$ and $\lambda$.
Instead of fixing $k$, we let \verb!GreedyTL! select features as long as the regularized error between two consecutive steps is larger than $\delta$.
%
\begin{figure*}[t]
  \caption{Baselines and number of additional noise dimensions sampled from a standard distribution.
    Averaged class-balanced recognition accuracies in the leave-one-class-out setting.}
  \centering
  \setlength{\subfiglength}{5.5cm}
  \subfloat[Caltech-256]{%
    \includegraphics[width=\subfiglength]{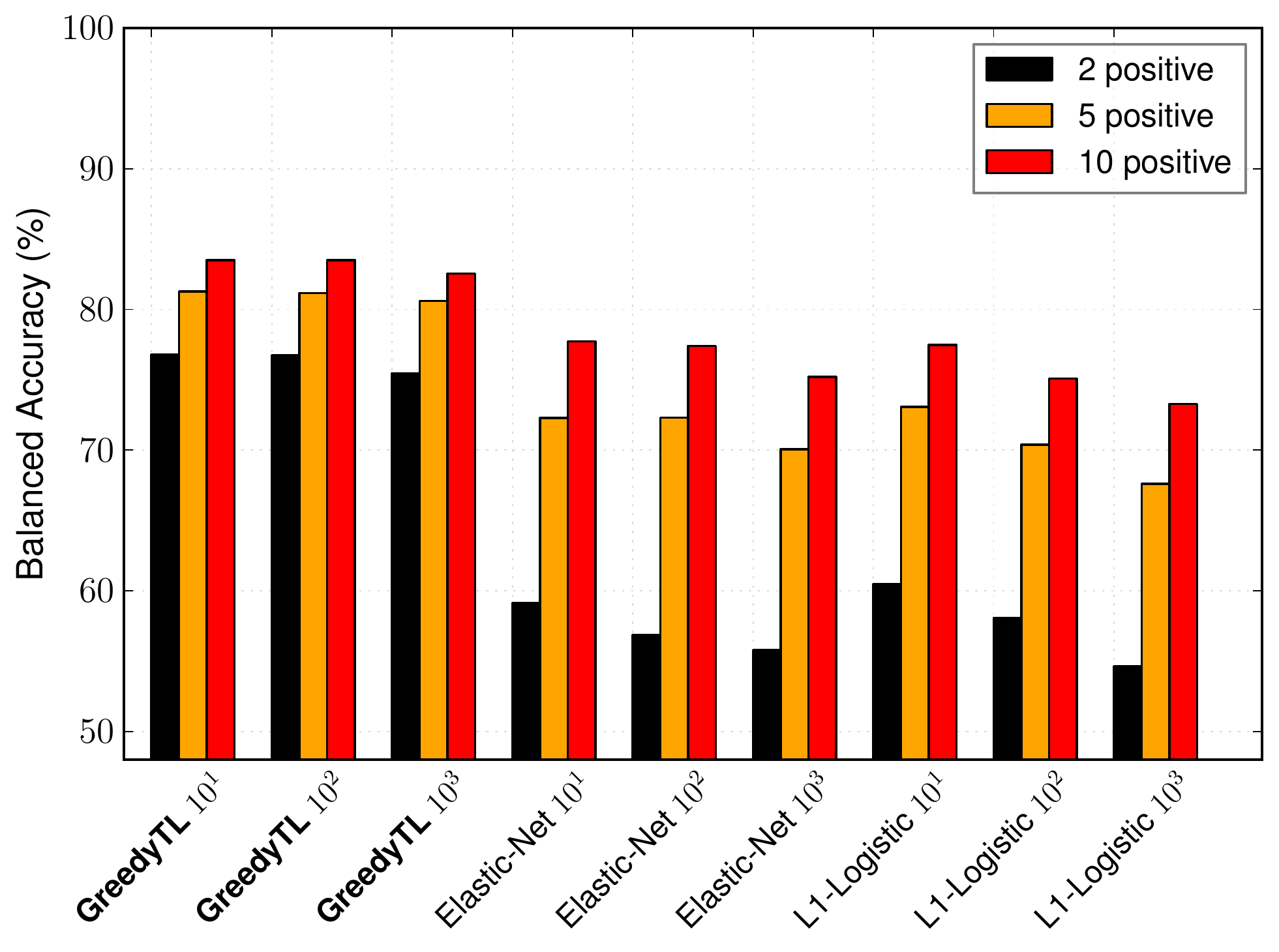}
    \label{fig:acc_caltech256_noise}
  }%
  \subfloat[Imagenet] {%
    \includegraphics[width=\subfiglength]{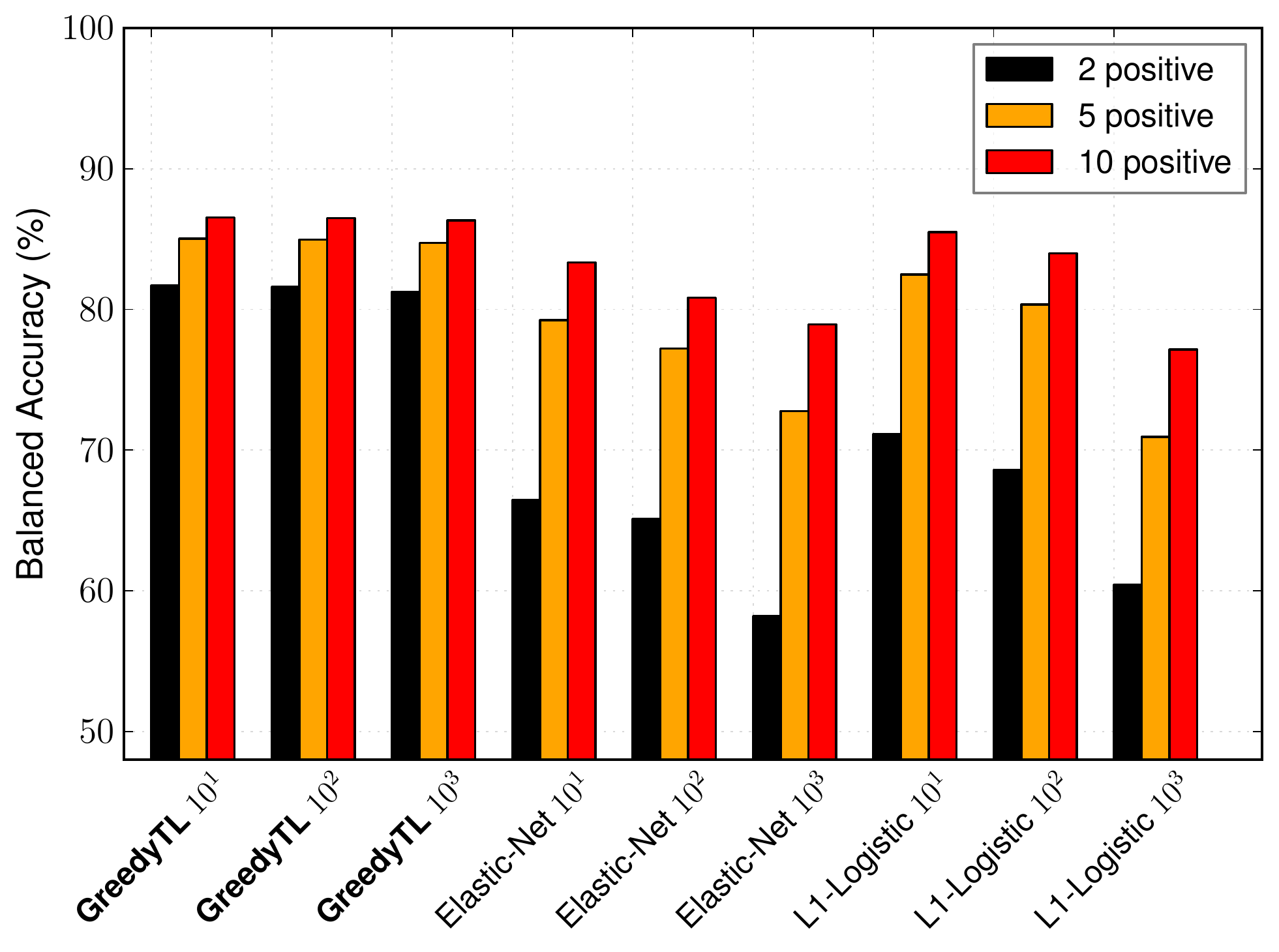}
    \label{fig:acc_imagenet_noise}
  }
\label{fig:accuracies_noise}
\vspace{-0.5cm}
\end{figure*}
%
In particular, we set $\delta = 10^{-4}$, as in preliminary experiments we have not observed any gain in performance past that point.
The $\lambda$ is fixed to $1$. 
Even better performance could be obtained tuning it.
%

%
%
%
%
We see that \verb!GreedyTL! dominates TL and feature selection baselines throughout the benchmark, rarely appearing on-par, especially in the small-sample regime.
In addition, on two datasets out of three, it manages to identify the source classifier subset that performs comparably or better than the Best source, that is the single best classifier selected by its performance on the testing set.
The significantly stronger performance achieved by \verb!GreedyTL! w.r.t. \ac{FR}, on all databases and in all settings, confirms the importance of the regularization in our formulation.

Notably, \verb!GreedyTL! outperforms RLS src+feat, which is equivalent to \verb!GreedyTL! selecting all the sources and features.
This observation points to the fact that \verb!GreedyTL! successfully manages to discard irrelevant feature dimensions and sources.
To investigate this important point further, we artificially add $10$, $100$ and $1000$ dimensions of pure noise sampled from a standard distribution.
Figure~\ref{fig:accuracies_noise} compares feature selection methods to \verb!GreedyTL! in robustness to noise.
Clearly, in the small-sample setting, \verb!GreedyTL! is tolerant to large amount of noise, while $L1$ and $L1/L2$ regularization suffer a considerable loss in performance.
We also draw attention to the failure of $L1$-based feature selection methods and MultiKT with $L1$ regularization to match the performance of \verb!GreedyTL!.
\subsection{Approximated GreedyTL}
\label{sec:exps_approximated_greedytl}
As was discussed in Section~\ref{tl}, the computational complexity of \verb!GreedyTL! is linear in the number of source hypotheses and feature dimensions. 
In this section we assess empirical performance of the approximated \verb!GreedyTL!,
which is \emph{independent} from the number of source hypotheses, implemented through the approximated greedy algorithm described at the end of Section~\ref{tl}.
In the following we refer to this version of an algorithm as \verb!GreedyTL-59!.
Instead of considering all the transfer learning and feature selection baselines, we restrict the performance comparison to the strongest competitors.
To show the power of highly scalable approximated \verb!GreedyTL!, we focus on the largest datasets in the number of source hypotheses and feature dimensions: Imagenet and SUN-$397$.
In case of Imagenet, we consider standard SIFT-BOW features as in previous section and also DeCAF-7 convolutional features extracted from the seventh layer of the DeCAF neural network~\cite{decaf}.
For the SUN-$397$, we use convolutional features of Caffe network trained on the Places-$205$ dataset~\cite{zhou2014learning}, which was shown to perform particularly well in the scene recognition tasks.
Figure~\ref{fig:new_accuracies} summarizes new results.
\begin{figure*}[!t]
  \caption{Comparison of the approximated GreedyTL: GreedyTL-59 to GreedyTL with exhaustive search and most competitive baselines on three largest datasets considered in our experiments.}
  \centering
  \setlength{\subfiglength}{4.8cm}
  \subfloat[Imagenet (SIFT-BOW, $1000$ classes)]{%
    \includegraphics[height=\subfiglength]{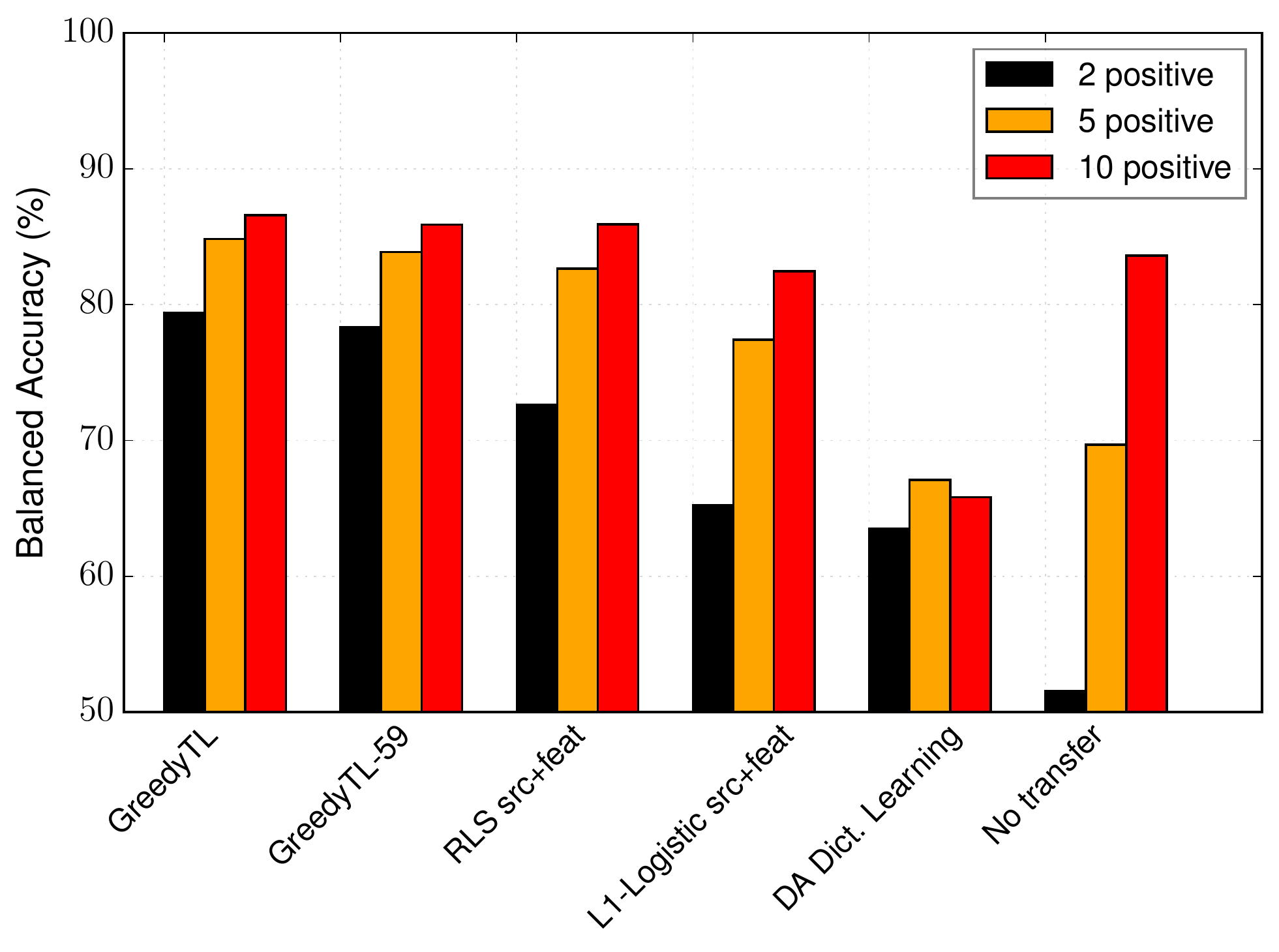}
    \label{fig:new_acc_ilsvrc_sift}
  }%
  \subfloat[Imagenet (DECAF-7 features, $1000$ classes)]{%
    \includegraphics[height=\subfiglength]{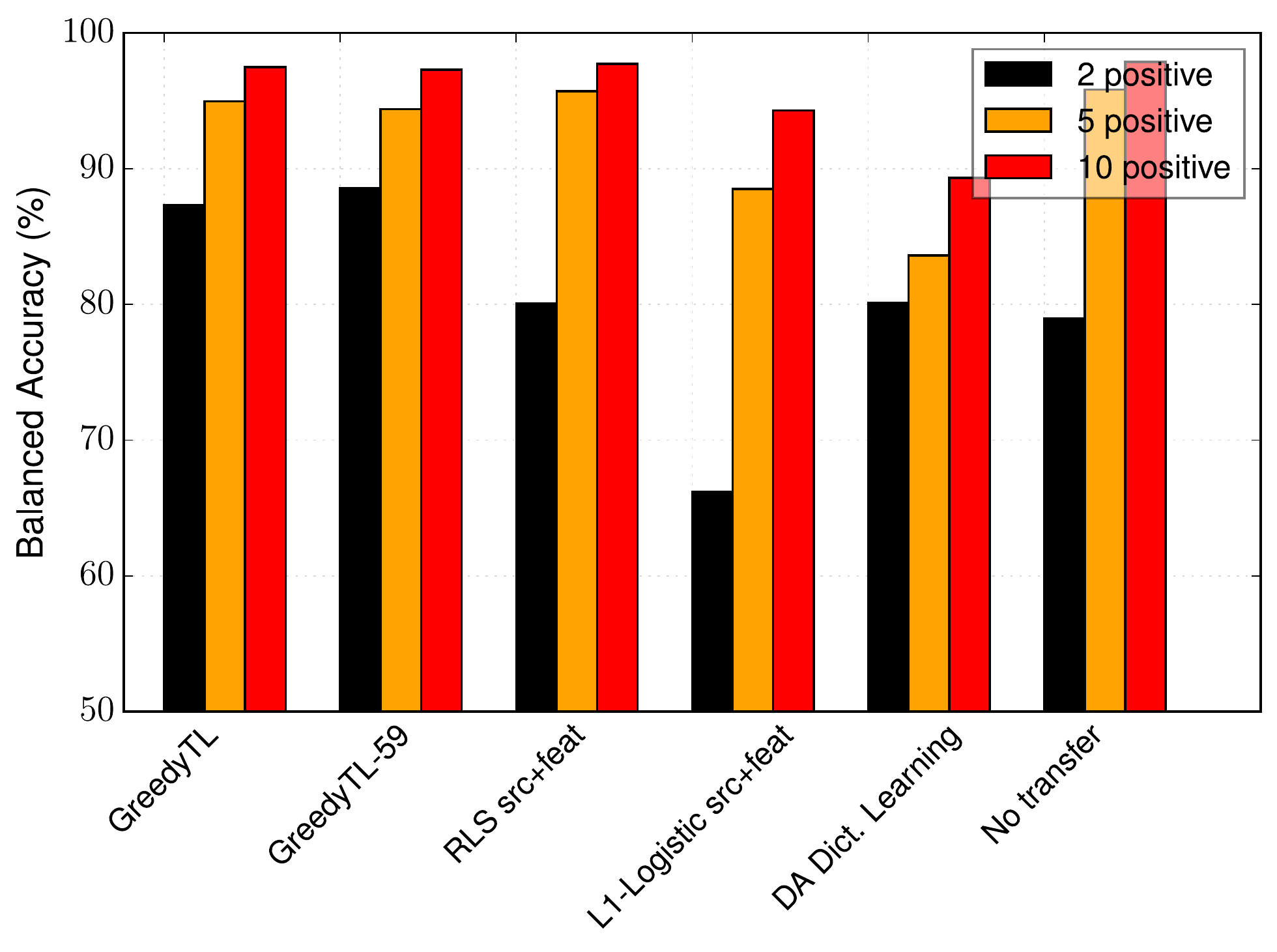}
    \label{fig:new_acc_ilsvrc_decaf}
  }\\
  \subfloat[SUN-$397$ (Caffe-7 features, $1000$ classes)]{%
    \includegraphics[height=\subfiglength]{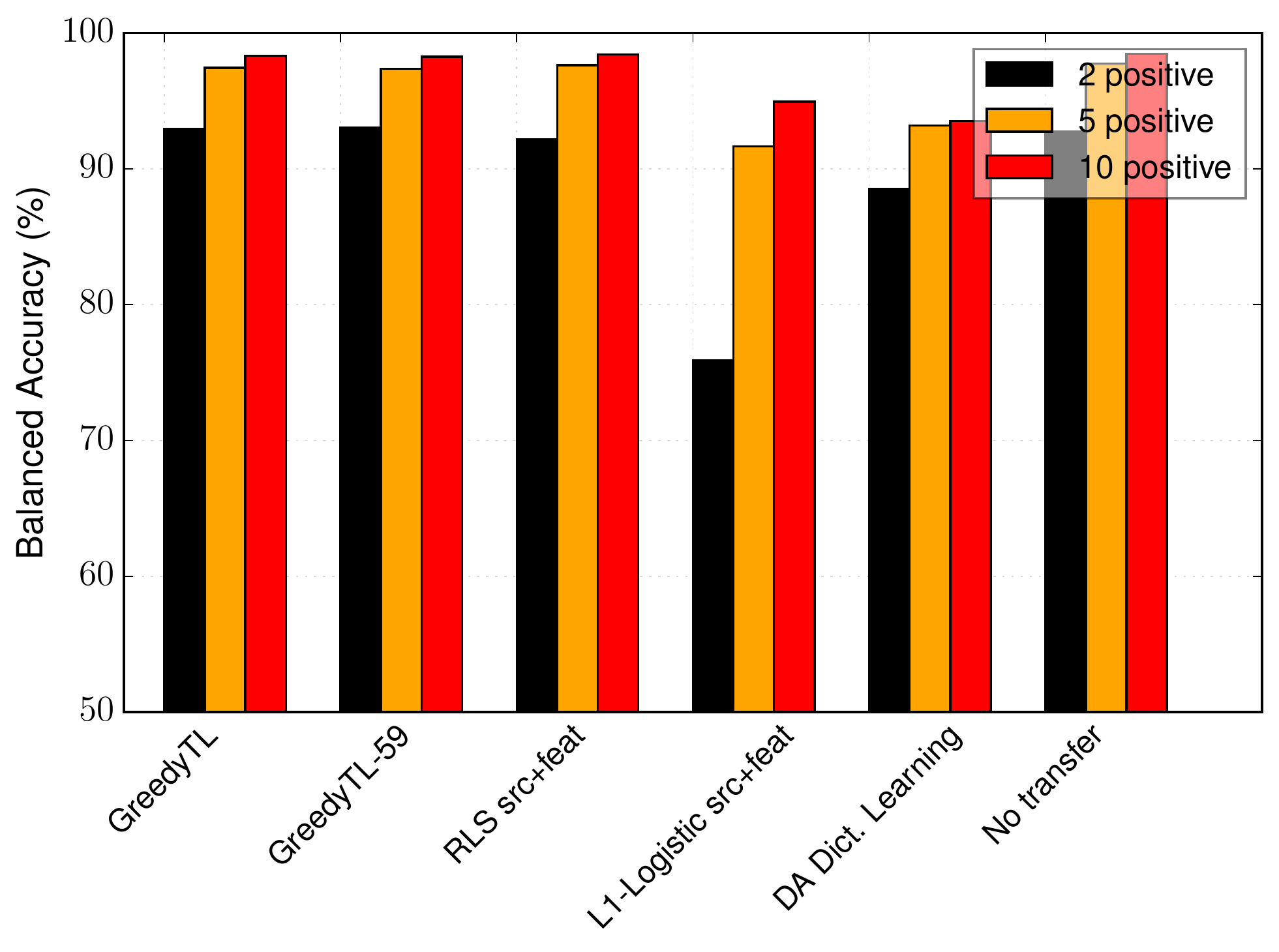}
    \label{fig:new_acc_sun397_caffe}
  }
\label{fig:new_accuracies}
\end{figure*}
Surprisingly, approximated \verb!GreedyTL! performs on par with the version with exhaustive search over the candidate, maintaining dominant performance in the small-sample regime on the Imagenet dataset.
Yet, training timings are dramatically improved as can be seen from Table~\ref{tab:timings}.
In the case of SUN-$397$ dataset, however, \verb!GreedyTL! performs on par with the top competitors.
%
%

\begin{table}[]
\scriptsize
\centering
\caption{Training time in seconds for transferring to a single target class.
Results are averaged over $10$ splits.}
\label{tab:timings}
\begin{tabular}{ll|l|l|l|}
\cline{3-5}
                                          & & GreedyTL~~~~~~~~~~~~~~~~& &                                                    \\ \hline
\multicolumn{2}{|l|}{Training examples pos.+neg.} & $2+10$       & $5+10$       & $10+10$                                    \\ \hline
\multicolumn{1}{|l|}{Imagenet (SIFT-BOW)} & $1899$ source+dim & $1.541 \pm 0.242$        & $3.083 \pm 0.486$        & $5.291 \pm 0.870$          \\ \hline
\multicolumn{1}{|l|}{Imagenet (DECAF7)}   & $4995$ source+dim & $3.481 \pm 0.356$        & $7.492 \pm 0.655$        & $13.408 \pm 1.165$         \\ \hline
\multicolumn{1}{|l|}{SUN-397 (Caffe-7)}   & $4492$ source+dim & $3.245 \pm 0.495$        & $6.764 \pm 1.051$        & $11.282 \pm 1.630$         \\ \hline
\end{tabular}

\begin{tabular}{ll|l|l|l|}
\cline{3-5}
                                          & & GreedyTL-59 & &                                                    \\ \hline
\multicolumn{2}{|l|}{Training examples pos.+neg.} & $2+10$       & $5+10$       & $10+10$                                    \\ \hline
\multicolumn{1}{|l|}{Imagenet (SIFT-BOW)}& $1899$ source+dim & $0.043 \pm 0.005$            & $0.088 \pm 0.011$           & $0.149 \pm 0.021$   \\ \hline
\multicolumn{1}{|l|}{Imagenet (DECAF7)}  & $4995$ source+dim & $0.055 \pm 0.006$            & $0.114 \pm 0.013$           & $0.198 \pm 0.020$  \\ \hline
\multicolumn{1}{|l|}{SUN-397 (Caffe-7)}  & $4492$ source+dim & $0.060 \pm 0.021$            & $0.120 \pm 0.038$           & $0.198 \pm 0.055$  \\ \hline
\end{tabular}
\end{table}

\subsection{Selected Source Analysis}
In this section we take a look at the source hypotheses selected by \verb!GreedyTL!. 
In particular, we make a qualitative assessment with the goal to see if semantically related sources and targets are correlated, visualizing selected sources and the magnitude of their weights.
We do so by grouping sources and targets semantically according to the WordNet~\cite{miller1995wordnet} distance, and plotting them as matrices with columns corresponding to targets, rows to sources, and entries to the weights of the sources.
Figure~\ref{lab:greedytl_ilsvrc_semantic} shows such matrices for \verb!GreedyTL! when evaluated on Imagenet with DECAF7 features and averaged over all splits, for $2$ positive and $10$ positive examples accordingly.
First we note, that for certain supercategories there are clearly distinctive patterns, indicating cross-transfer within the same supercategory.
\begin{figure*}[!t]
  \caption{Semantic transferrability matrix for GreedyTL evaluated on Imagenet (DECAF7 features). Columns correspond to targets and rows to sources. Stronger color intensity means larger source weight. \ref{fig:greedytl_ilsvrc_semantic_1} corresponds to learning from $2$ positive and $10$ negative examples, while~\ref{fig:greedytl_ilsvrc_semantic_2}, with $10$ positive.b}
  \label{lab:greedytl_ilsvrc_semantic}
  \centering
  \setlength{\subfiglength}{5.7cm}
  \subfloat[]{
    \includegraphics[height=\subfiglength]{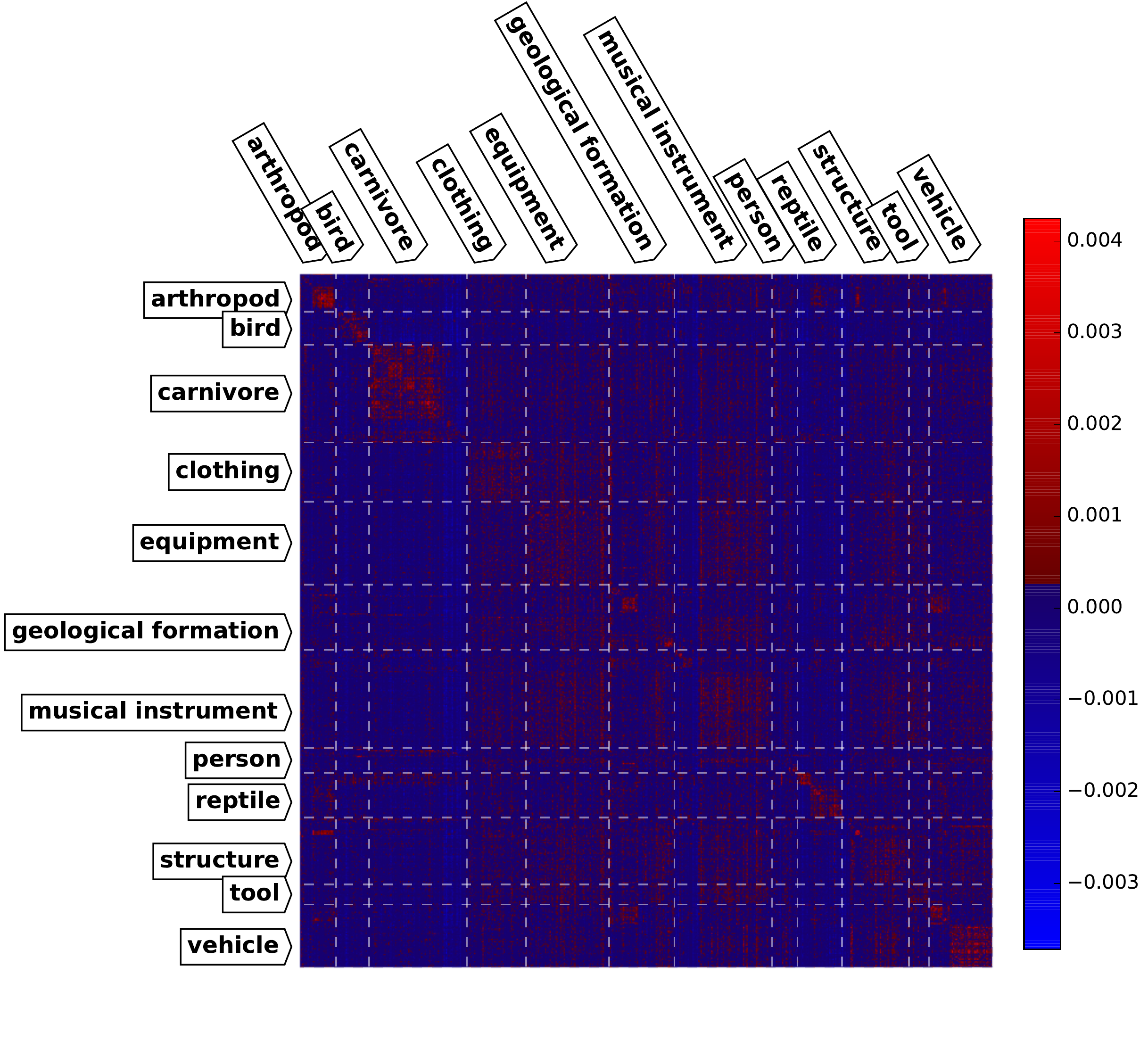}
    \label{fig:greedytl_ilsvrc_semantic_1}
  }
  \subfloat[]{
    \includegraphics[height=\subfiglength]{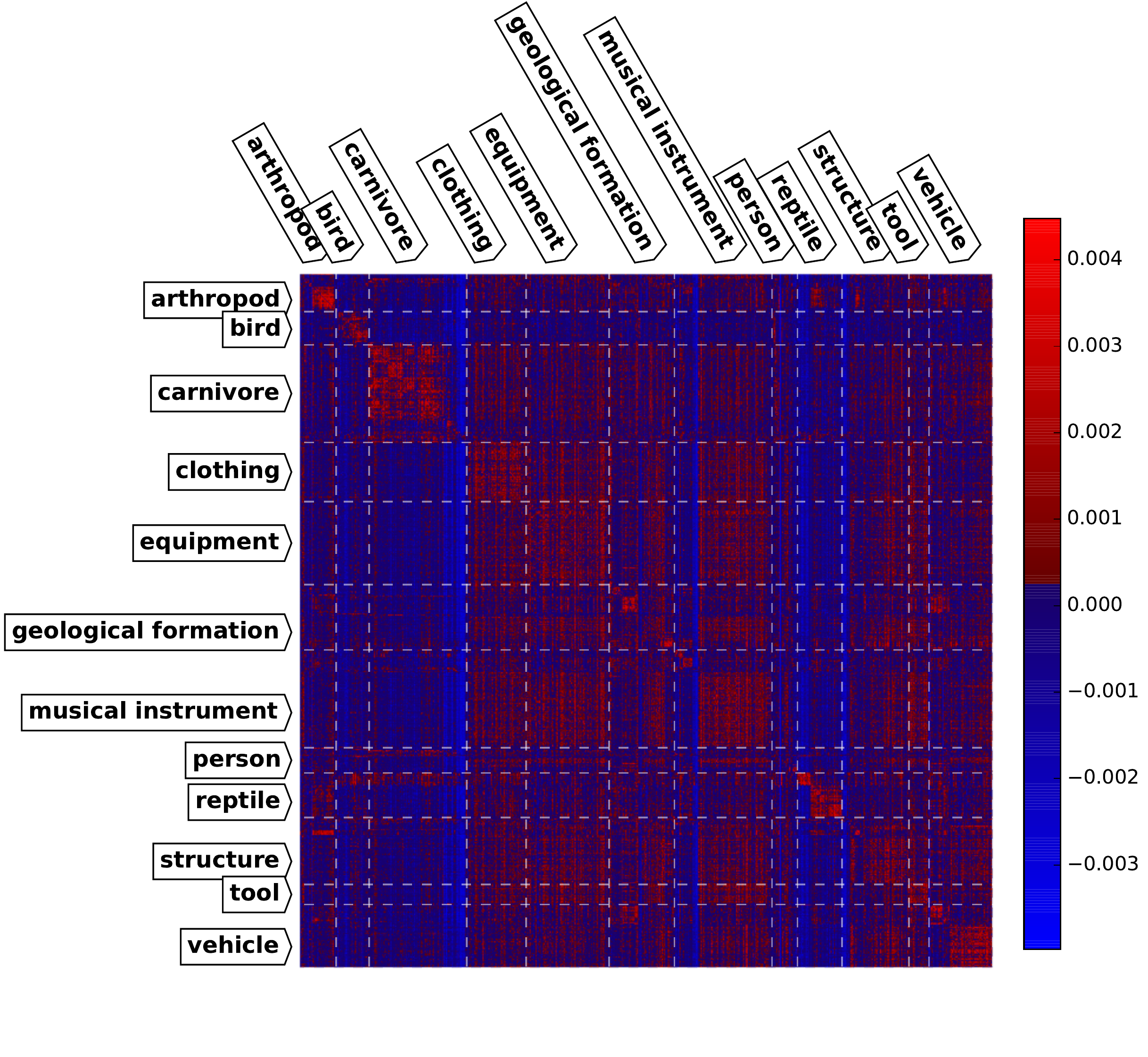}
    \label{fig:greedytl_ilsvrc_semantic_2}
  }\\
\end{figure*}
We compare those matrices to the ones originating from the strongest \verb!RLS (src+feat)! baseline, Figure~\ref{fig:rls_ilsvrc_semantic}.
We notice a clear difference, as semantic patterns of \verb!GreedyTL! are more distinctive in a small-sample setting (2+10), while the ones of \verb!RLS (src+feat)! appear hazier.
We argue that this is a consequence of greedy selection procedure implemented by \verb!GreedyTL!, where sources are selected incrementally, thus many coefficients correspond to zeros.
Due to the formulation of \verb!RLS (src+feat)!, however, even if a source is less relevant, its coefficient most likely will not be exactly equal to zero.
%
\begin{figure*}[!t]
  \caption{Semantic transferrability matrix for RLS (src+feat) evaluated on Imagenet (DECAF7 features).}
  \label{fig:rls_ilsvrc_semantic}
  \centering
  \setlength{\subfiglength}{5.7cm}
  \subfloat[]{
    \includegraphics[height=\subfiglength]{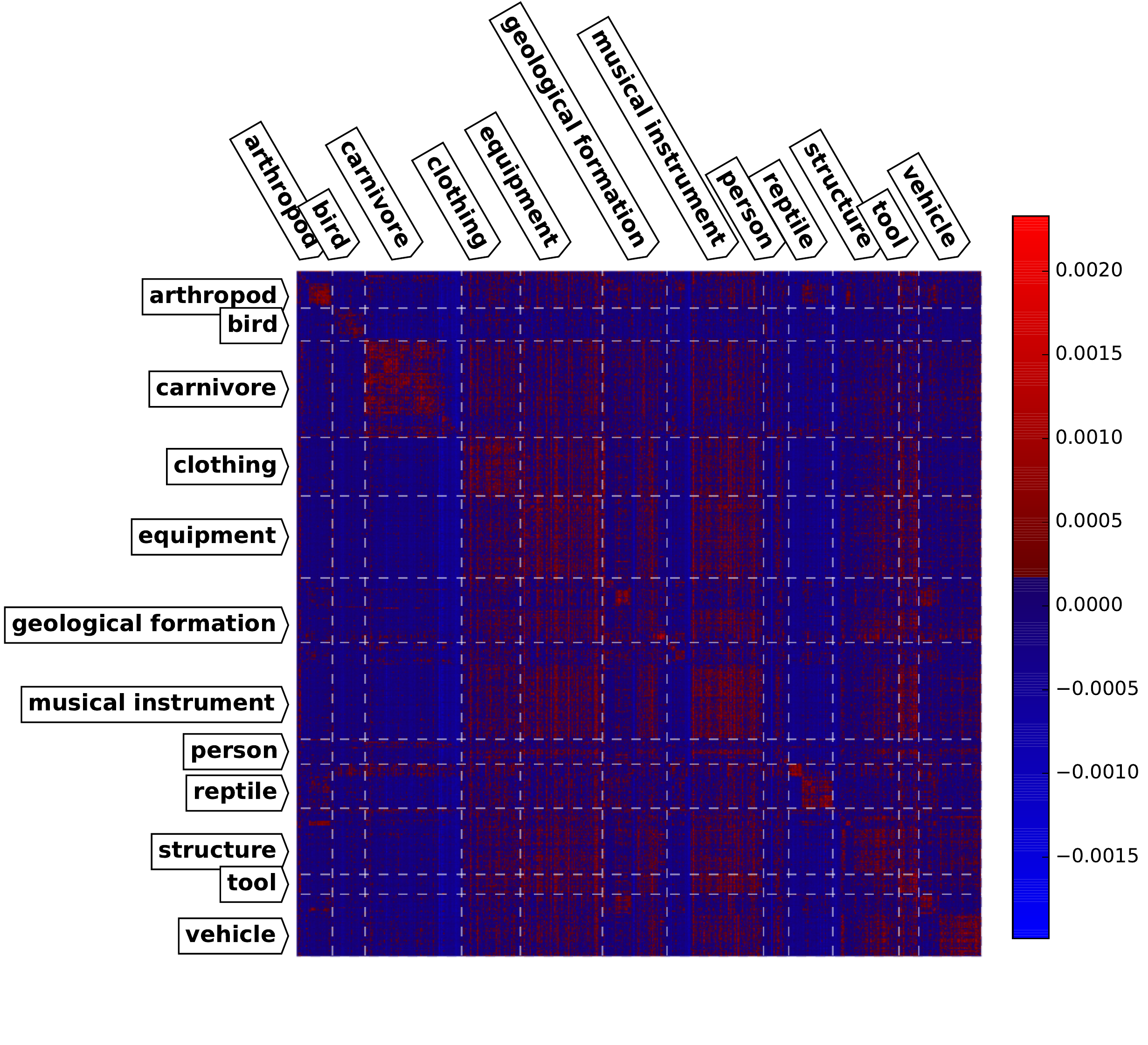}
  }%
  \subfloat[]{
    \includegraphics[height=\subfiglength]{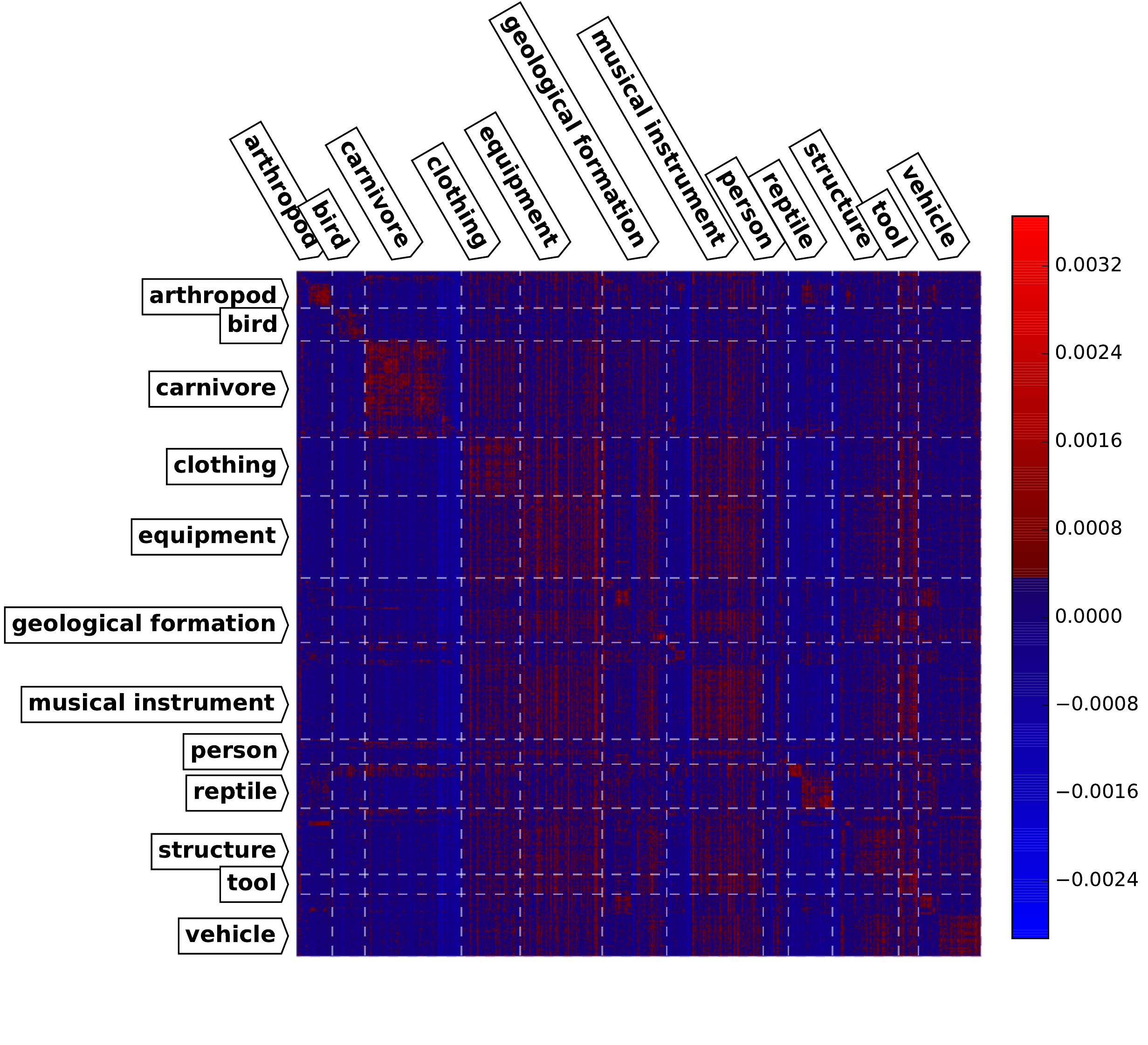}
  }\\
\end{figure*}

It is also instructive to compare exact \verb!GreedyTL! to the approximated one.
Figure~\ref{fig:greedytl59_semantic} pictures semantic matrices for the approximated version.
We note that
approximated version appears to be slightly more conservative in a small-sample case (2+10), but in overall, semantic patterns seem to match,
thus emphasizing the quality of the solution provided by the approximated version and empirically corroborating the theoretical motivation behind the randomized selection.
\begin{figure*}[!t]
  \caption{GreedyTL evaluated on Imagenet (DECAF7 features): a closer look at some strongly related sources and targets.}
  \label{fig:greedytl_semantic_zoom}
  \centering
  \includegraphics[height=12cm]{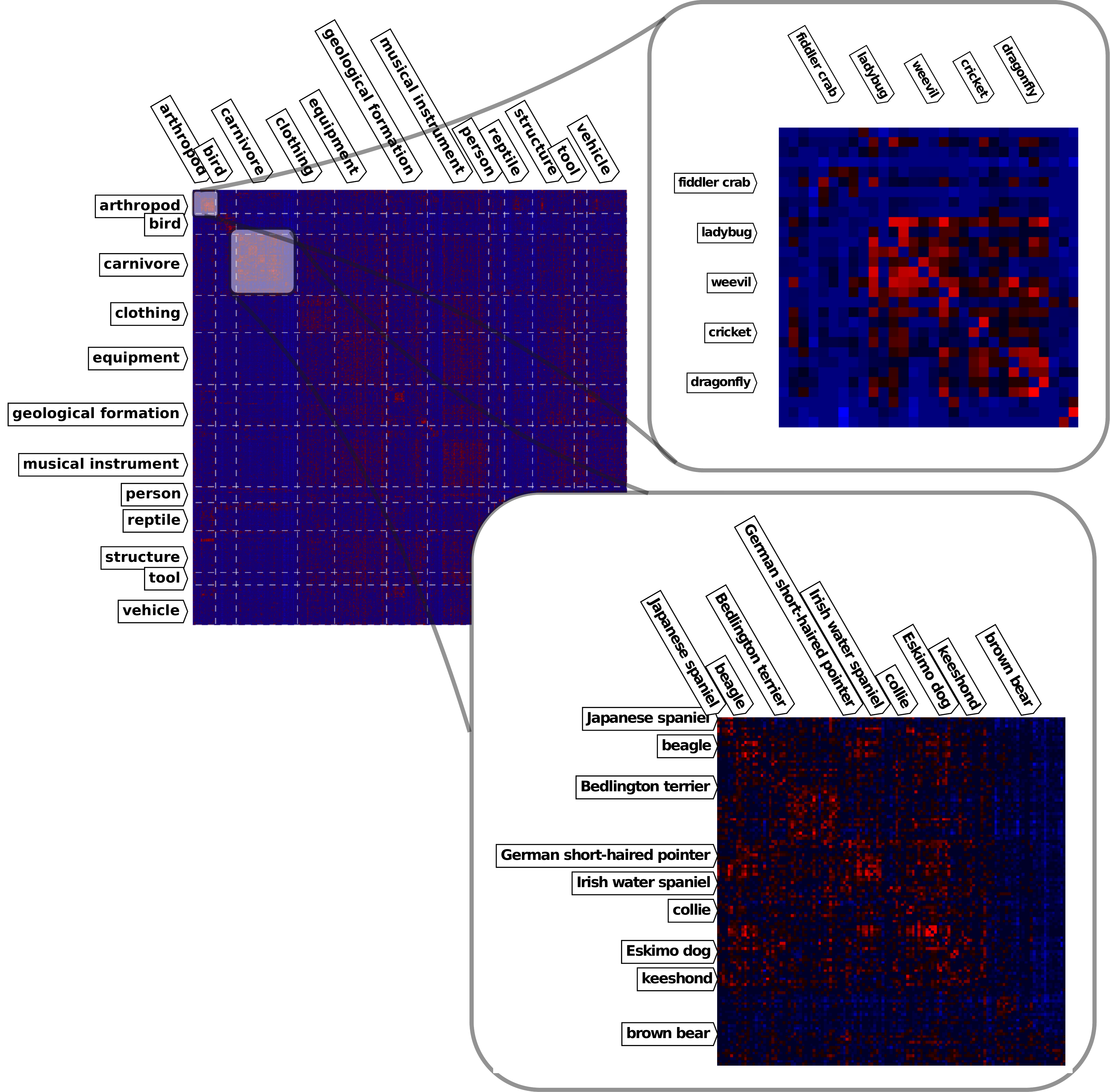}
\end{figure*}
\begin{figure*}[!t]
  \caption{Semantic transferrability matrix for the approximated GreedyTL evaluated on Imagenet (DECAF7 features).}
  \label{fig:greedytl59_semantic}
  \centering
  \setlength{\subfiglength}{5.7cm}
  \subfloat[]{
    \includegraphics[height=\subfiglength]{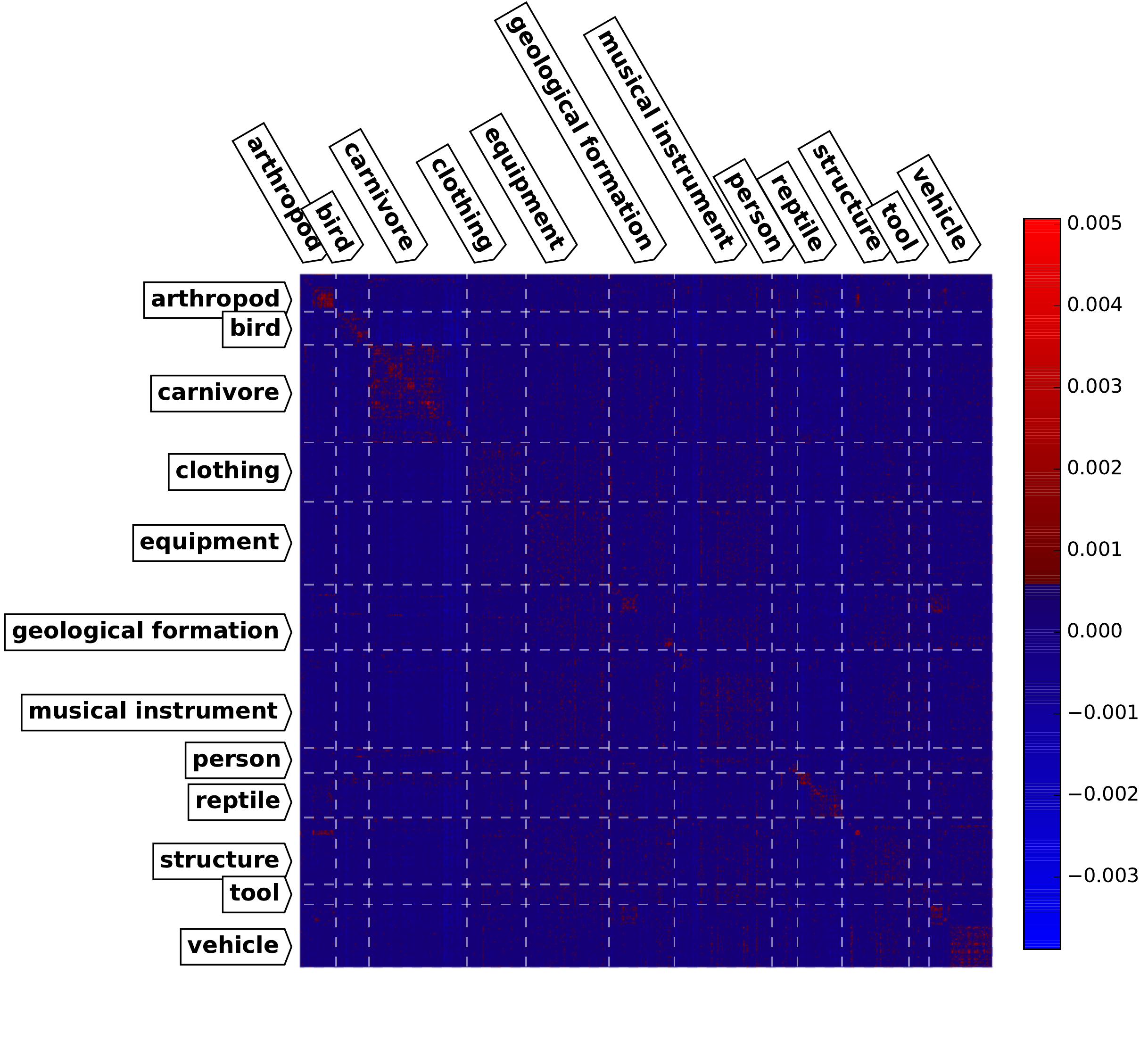}
  }%
  \subfloat[]{
    \includegraphics[height=\subfiglength]{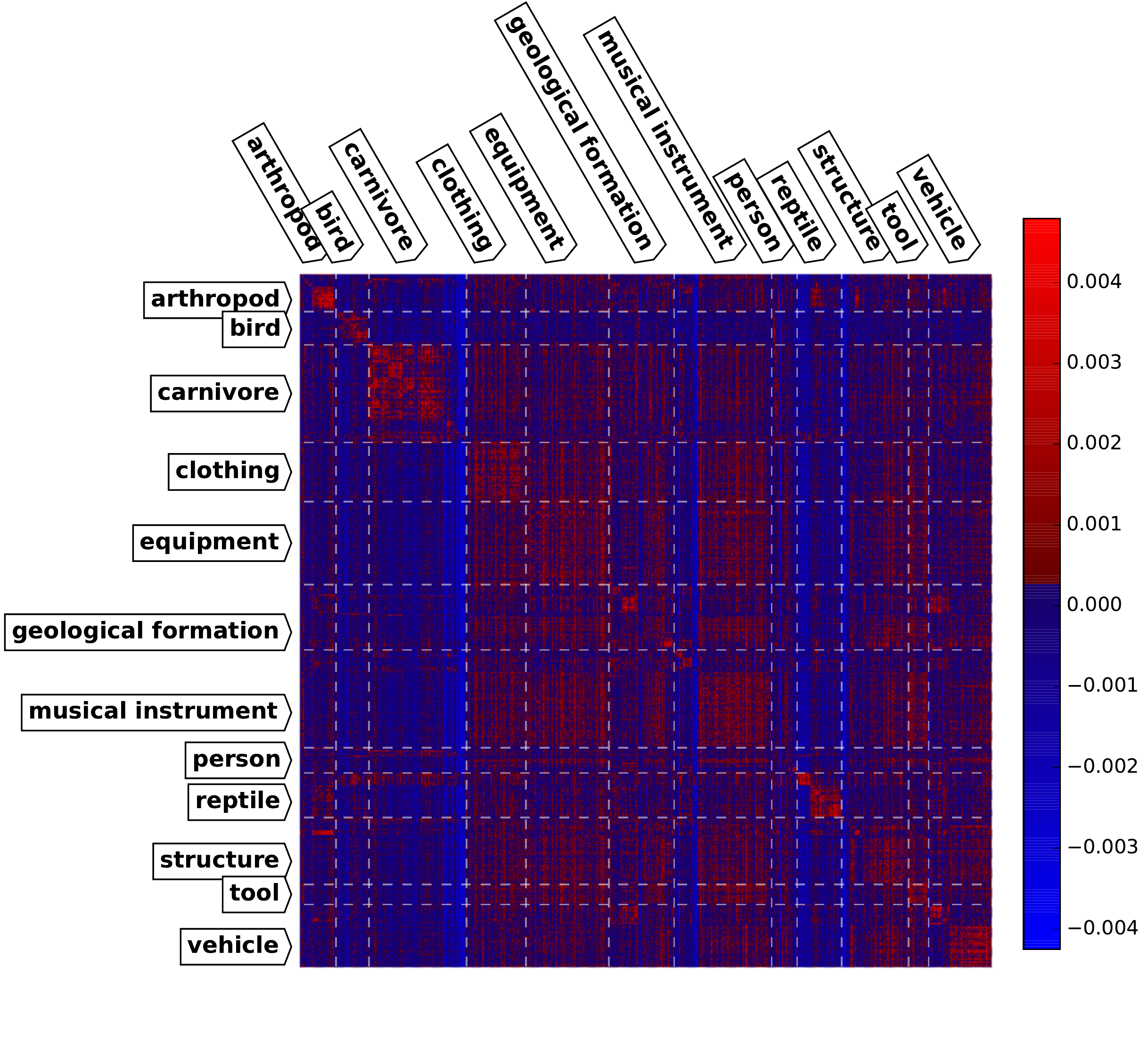}
  }\\
\end{figure*}
Finally, we take a closer look at some patterns of Figure ~\ref{fig:greedytl_ilsvrc_semantic_1}, that is in the case of learning from only $2$ positive examples. This new analysis is shown  in Figure~\ref{fig:greedytl_semantic_zoom}.
We notice that even at the smaller scale, there are emergent semantic patterns.

\vspace{-0.2cm}
\section{Conclusions}
\vspace{-0.1cm}
In this work we studied the transfer learning problem involving hundreds of sources.
The kind of transfer learning scenario we consider assumes no access to the source data directly, but through the use of the source hypotheses induced by them.
In particular, we focused on the efficient source hypothesis selection and combination, improving the performance on the target task.
We proposed a greedy algorithm, \verb!GreedyTL!, capable of selecting relevant sources and feature dimensions at the same time.
We verified these claims by obtaining the best results
among the competing feature selection and TL algorithms,
on the Imagenet, SUN09 and Caltech-256 datasets. At the same time, comparison against the non-regularized version of the algorithm clearly show the power of our intuition.
We support our empirical findings by
showing theoretically that under reasonable assumptions on the sources, the algorithm can learn effectively from few target examples.


%
\vspace{-0.5cm}
 \section*{Acknowledgments}
This work was partially supported by the ERC grant 367076 -RoboExNovo (B.C. and I. K.).
{\small
\bibliographystyle{unsrt}
\bibliography{learning}
}

\appendix
\section{Proofs}
In this section we present proofs of theorems.
%
For brevity, we define
$\bhsrc(\bx) := [\hsrc_1(\bx), \ldots, \hsrc_n(\bx)]\tp$,
and we will consider a truncated target predictor
\[
\htrg_{\bw, \bbeta}(\bx) := \scT\left(\bw\tp \bx + \bbeta\tp \bhsrc(\bx)\right)~,
\] with $\scT(a) := \min\{\max\{a, -1\}, 1\}$.
That said, we will assume that
\[
\Riskh(\htrg_{\bw, \bbeta}) \leq \frac{1}{m} \sum_{i=1}^m (\bw\tp \bx_i + \bbeta\tp \bhsrc(\bx_i) - y_i)^2~,
\] in other words, empirical risk of truncated predictor cannot be greater, since all the labels belong to $\{-1,1\}$.

To prove Theorem~\ref{thm:gen_no_approx} we need the following supplementary lemmas.
\begin{lemma}
  \label{lem:norm_bound}
  Let \verb!GreedyTL! generate solution $(\bhatw, \bhatbeta)$, given the training set $(\bX, \by)$, source hypotheses $\{\hsrc_i\}_{i=1}^n$, and hyperparameters $\lambda$ and $k$.
  Then we have that,
  \begin{equation*}
    \lambda \|\bhatw\|^2 + \lambda \|\bhatbeta\|^2 + \Riskh(\htrg_{\bhatw, \bhatbeta})  \leq \min_{|S| \leq k} \left\{ \frac{1}{|S|} \sum_{j \in S} \Riskh(\hsrc_j) + \frac{\lambda}{|S|} \right\}~,
  \end{equation*}
  \begin{equation*}
    \lambda \|\bhatw\|^2 + \lambda \|\bhatbeta\|^2 + \Riskh(\htrg_{\bhatw, \bhatbeta})  \leq \Riskh(\htrg_{\bzero, \bhatbeta})~.
  \end{equation*}
  and also,
  \[
  \lambda \|\bhatw\|^2 + \lambda \|\bhatbeta\|^2 + \Riskh(\htrg_{\bhatw, \bhatbeta}) \leq 1~.
  \]
\begin{proof}
    Define $J(\bw, \bbeta) := \Riskh(\htrg_{\bw, \bbeta}) + \lambda \|\bw\|^2 + \lambda \|\bbeta\|^2$.
    For any $\balpha \in \left\{0, \frac{1}{p}\right\}^n$ such that $\|\balpha\|_0 = p$ we have,
    {\small
    \begin{align}
      &J(\bhatw, \bhatbeta) \leq  J(\bzero, \balpha)
      = \frac{1}{m}\sum_{i=1}^m \ell\left(y_i, \frac{1}{p} \sum_{j \in \supp(\balpha)} \hsrc_j(\bx_i) \right) + \frac{\lambda}{p} \nonumber\\
      &\leq \frac{1}{p} \sum_{j \in \supp(\balpha)} \Riskh(\hsrc_j) + \frac{\lambda}{p}~\label{eq:norm_bound}.
    \end{align}
  }
    We have the last inequality due to Jensen's inequality.
    The fact that~\eqref{eq:norm_bound} holds for any $p \in \{1, \ldots, k\}$ proves the first statement.

    We have the second statement from,
    \begin{align*}
      &\Riskh(\htrg_{\bhatw, \bhatbeta}) + \lambda \|\bhatw\|^2 + \lambda \|\bhatbeta\|^2 \leq \Riskh(\htrg_{\bzero, \bhatbeta}) + \lambda \|\bhatbeta\|^2\\
      \Rightarrow &\Riskh(\htrg_{\bhatw, \bhatbeta}) \leq \Riskh(\htrg_{\bhatw, \bhatbeta}) + \lambda \|\bhatw\|^2 \leq \Riskh(\htrg_{\bzero, \bhatbeta}).
    \end{align*}

    The last statement comes from,
    \begin{align}
      \lambda \|\bhatw\|^2 + \lambda \|\bhatbeta\|^2 \leq  J(\bzero, \bzero) \leq 1~.
    \end{align}
\end{proof}
\end{lemma}
\begin{lemma}
  \label{lem:norm_bound_opt}
  Let $(\bw^\star, \bbeta^\star)$ be the optimal solution to~\eqref{def:htl_subset_selection}, given the training set $(\bX, \by)$, source hypotheses $\{\hsrc_i\}_{i=1}^n$, and hyperparameters $\lambda$ and $k$.
  Then, the following holds,
  \begin{align*}
    &\lambda \|\bw^\star\|^2 + \lambda \|\bbeta^\star\|^2 + \Riskh(\htrg_{\bw^\star, \bbeta^\star})\\
    &\quad\leq \min_{|S| \leq k} \left\{ \frac{1}{|S|} \sum_{j \in S} \Riskh(\hsrc_j) + \frac{\lambda}{|S|} \right\}~.
  \end{align*}
\end{lemma}
\begin{proof}
    Define $J(\bw, \bbeta) := \Riskh(\htrg_{\bw, \bbeta}) + \lambda \|\bw\|^2 + \lambda \|\bbeta\|^2$.
    For any $\balpha \in \left\{0, \frac{1}{p}\right\}^n$ such that $\|\balpha\|_0 = p$ we have,
    {\small
    \begin{align}
      &J(\bw^\star, \bbeta^\star) \leq  J(\bzero, \balpha)
      = \frac{1}{m}\sum_{i=1}^m \ell\left(y_i, \frac{1}{p} \sum_{j \in \supp(\balpha)} \hsrc_j(\bx_i) \right) + \frac{\lambda}{p} \nonumber\\
      &\leq \frac{1}{p} \sum_{j \in \supp(\balpha)} \Riskh(\hsrc_j) + \frac{\lambda}{p}~\label{eq:norm_bound}.
    \end{align}
  }
    We have the last inequality due to Jensen's inequality.
    The fact that~\eqref{eq:norm_bound} holds for any $p \in \{1, \ldots, k\}$ proves the statement.

\end{proof}
%
\begin{proof}[Proof of Theorem \ref{thm:gen_no_approx}]
To prove the statement we will use the optimistic rate Rademacher complexity bounds of~\cite{SrebroST10}.
In particular, we will have to do two things: upper-bound the worst-case Rademacher complexity of the hypothesis class of \verb!GreedyTL!,
and upper-bound the empirical risk of members of that hypothesis class.
Before proceeding, we spend a moment to define the loss class of \verb!GreedyTL!, assuring that it is consistent with the definition by~\cite{SrebroST10},
{\small
\begin{equation}
\sL := \left\{ (\bx, y) \mapsto  \frac{1}{2} \left(h(\bx) - y\right)^2 : h \in (\scT \circ \sH), \ \Riskh(h) \leq r \right\}~.~\label{eq:erm_loss_class}
\end{equation}
}
Here, $(\scT \circ \sH)$ is the class of truncated hypotheses, $\sH$ is the hypothesis class of \verb!GreedyTL! and $r$ is the mentioned bound on the empirical risk.
We define the hypothesis class as,
\[
\sH := \left\{ \bx \mapsto \bw\tp \bx + \bbeta\tp \bhsrc(\bx) : \|\bw\|^2_2 + \|\bbeta\|^2_2 \leq \frac{1}{\lambda} \right\}~.
\]
In this definition we have used the fact shown in Lemma~\ref{lem:norm_bound}, that is the constraint on $\|\bw\|_2^2 + \|\bbeta\|_2^2$, which translates into a constraint on the hypothesis class.
Now we are ready to analyze its complexity.

Recall that the worst case Rademacher complexity is defined as,
\begin{equation*}
\Rad(\sF) := \sup_{\bx_1, \ldots, \bx_m \in \sX} \left\{ \E_{\bsigma}\left[ \sup_{f \in \sF}\left\{ \frac{1}{m} \sum_{i=1}^m \sigma_i f(\bx_i) \right\} \right] \right\}~,
\end{equation*}
where $\sigma_i$ is r.v., such that $\mathbb{P}(\sigma_i=1) = \mathbb{P}(\sigma_i=-1) = \frac{1}{2}$.

Let us focus on the analysis of empirical Rademacher complexity $\Radh(\scT \circ \sH)$, that is the part inside the outer supremum.
The truncation $\scT()$  is $1$-Lipschitz, therefore by Talagrand's contraction lemma~\cite{mohri2012foundations} we have that $\Radh(\scT \circ \sH) \leq \Radh(\sH)$.
Hence, now we proceed with an upper-bound on $\Radh(\sH)$.
Define $\biota \in \{0,1\}^n$ such that $\iota_i := \bigg\{ \begin{array}{cc}1,~~ & i \in \supp(\bbeta) \\ 0,~~ & \text{otherwise} \end{array}$.
Then we have that,
\begin{align}
&\Radh(\scT \circ \sH) \leq \Radh(\sH)\\
&\quad= \MEbr{\bsigma}{\sup_{ \|\bw\|^2_2 + \|\bbeta\|^2_2 \leq \frac{1}{\lambda} } \frac{1}{m} \sum_{i=1}^m \sigma_i ( \bw\tp \bx_i + \bbeta\tp \bhsrc(\bx_i) ) } \nonumber\\
&\quad = \frac{1}{m \sqrt{\lambda}} \MEbr{\bsigma}{\left\| \sum_{i=1}^m \sigma_i \left[\begin{array}{c}\bx_i\\\biota \circ \bhsrc(\bx_i)\end{array}\right] \right\|} \label{eq:radh_cs}\\
&\quad \leq \sqrt{\frac{1}{m^2 \lambda} \sum_{i=1}^m \|\bx_i\|^2 + \|\biota \circ \bhsrc(\bx_i)\|^2} \label{eq:jensen_sigma}\\
&\quad \leq \sqrt{\frac{1 + k \|\bhsrc\|^2_\infty}{\lambda m}  }~. \label{eq:radh_last}
\end{align}
To obtain~\eqref{eq:radh_cs} we have applied Cauchy-Schwartz inequality on the inner product of $[\bw\tp~ \bbeta\tp]\tp$ and $[\bx_i\tp~ \bhsrc(\bx_i)\tp]\tp$, then upper-bounding norms with constraints given by definition of a class $\sH$.
To get~\eqref{eq:jensen_sigma} we have applied Jensen's inequality w.r.t. $\MEbr{}{\cdot}$, along with the fact that $\MEbr{}{\sigma_i \sigma_{j\neq i}} = 0$ and $\MEbr{}{\sigma_i \sigma_i} = 1$.
Next, we have bounded the $L2$ norms of features and sources, recalling that by assumption, $\|\bx_i\|^2 \leq 1$.
Finally, taking supremum over~\eqref{eq:radh_last} w.r.t. data, we obtain,
\[
\Rad(\scT \circ \sH) \leq \sqrt{\frac{1 + k \|\bhsrc\|^2_\infty}{\lambda m}  }~.
\]

Next, we upper bound the empirical risk of the members of $\sH$ by Lemma~\ref{lem:norm_bound}.
By plugging the bound on the $\Rad(\sH)$, and the bound on the empirical risk of~\eqref{eq:erm_loss_class} into Theorem~1 in~\cite{SrebroST10} we have the statement. 
\end{proof}

Next we prove the approximation guarantee of a \ac{RSS}, Corollary~\ref{cor:tikhonov_fr_approx}, that is needed for proof of Theorem~\ref{thm:gen_approx}.
First we note that the solution returned by~\ac{FR} enjoys the following guarantees in solving the Subset Selection.
\begin{theorem}[\cite{das2008algorithms}]
\label{thm:fr_approx}
Assume that $\bC$ and $\bb$ are normalized,
and $C_{i,j\neq i} \leq \gamma < \frac{1}{6 k}$ for subset size $k \leq n$.
Then, the \ac{FR} algorithm generates an approximate solution $\bhatw$ to the Subset Selection such that,
$\Risk(\bhatw) \leq (1 + 16(k + 1)^2 \gamma) \min_{\|\bw\|_0 = k} \Risk(\bw)~.$
\end{theorem}
This theorem is instrumental in stating our corollary.
%

\begin{proof}[Proof of Corollary~\ref{cor:tikhonov_fr_approx}]
  In addition to the sample coviance matrix $\bhatC$, define also correlations $\bb := \frac{1}{m} \bX\tp \by$.
  Denote $\bhatC' = \frac{\bhatC + \lambda \bI}{1 + \lambda}$.
  Now, suppose that $\bhatw_S$ is the solution found by the forward regression algorithm, given the input $(\bhatC', \bhatb, k)$.
  So, the empirical risk that the algorithm attains is $1 - \bhatb_S\tp (\bhatC_S')^{-1} \bhatb_S$, as follows from the analytic solution to empirical risk minimization for given $S$.
  In fact, we can upper-bound it right away using Theorem~\ref{thm:fr_approx}.
  But, recall that our goal is to upper-bound the quantity $\Riskh(\bhatw) + \lambda \|\bhatw\|^2 = 1 - \bhatb_S \tp (\bhatC_S + \lambda \bI)^{-1} \bhatb_S$, that is the regularized empirical risk of the approximation $\bhatw_S$ to the \ac{RSS}.
  This quantity is obtained via the unnormalized covariance matrix, therefore we cannot analyze it directly by Theorem~\ref{thm:fr_approx}.
  For this reason we rewrite it as $\Riskh(\bhatw) + \lambda \|\bhatw\|^2 = 1 - \frac{1}{1 + \lambda}\bhatb_S\tp\left(\bhatC_S'\right)^{-1} \bhatb_S$.
  From Theorem~\ref{thm:fr_approx} assumptions we then have $(\hat{C}_S)'_{i,j\neq i} \leq \gamma' \leq \frac{1}{6k}$, denote $\epsilon = 16(k+1)^2 \gamma'$, and let $S^\star$ be the optimal subset of size $k$.
  Now we plug $1 - \bhatb_S\tp (\bhatC_S')^{-1} \bhatb_S$ into Theorem~\ref{thm:fr_approx}, and proceed with algebraic transformations,
  \begin{align}
    &1 - \bhatb_S\tp (\bhatC_S')^{-1} \bhatb_S \leq (1 + \epsilon) (1 - \bhatb_{S^\star}\tp (\bhatC_{S^\star}')^{-1} \bhatb_{S^\star}) \nonumber\\
    &\Rightarrow \frac{1}{1 + \lambda} ( 1 - \bhatb_S\tp (\bhatC_S')^{-1} \bhatb_S ) \leq \frac{1 + \epsilon}{1 + \lambda} (1 - \bhatb_{S^\star}\tp (\bhatC_{S^\star}')^{-1} \bhatb_{S^\star}) \nonumber\\
    &\Rightarrow 1 - \frac{1}{1 + \lambda} \bhatb_S\tp (\bhatC_S')^{-1} \bhatb_S\\
    &\quad \leq (1 + \epsilon) \left(\frac{1}{1 + \lambda} - \frac{1}{1 + \lambda} \bhatb_{S^\star}\tp (\bhatC_{S^\star}')^{-1} \bhatb_{S^\star}\right) + \frac{\lambda}{1 + \lambda} \nonumber \\
    &\Rightarrow 1 - \frac{1}{1 + \lambda} \bhatb_S\tp (\bhatC_S')^{-1} \bhatb_S\\
    &\quad \leq (1 + \epsilon) \left(1 - \frac{1}{1 + \lambda} \bhatb_{S^\star}\tp (\bhatC_{S^\star}')^{-1} \bhatb_{S^\star}\right) - \frac{\epsilon \lambda}{1 + \lambda}~.\nonumber
  \end{align}
%
  The last step is to relate $\gamma'$ to $\gamma$.
  The fact $(\hat{C}_S)'_{i,j\neq i} \leq \gamma' \leq \frac{1}{6k}$ is equivalent to $\frac{(\hat{C}_S)_{i,j\neq i}}{1 + \lambda} \leq \gamma' \leq \frac{1}{6k}$.
  Therefore we can set $\gamma = \gamma' (1+\lambda)$ and obtain $(\hat{C}_S)_{i,j\neq i} \leq \gamma \leq \frac{1 + \lambda}{6k}$.
  This concludes the proof.
  \end{proof}

\begin{proof}[Proof of Theorem~\ref{thm:gen_approx}]
The proof follows the composition of Theorem~\ref{thm:gen_no_approx}, Corollary~\ref{cor:tikhonov_fr_approx} and Lemma~\ref{lem:norm_bound_opt}.
In particular, we upper-bound the empirical risk of Theorem~\ref{thm:gen_no_approx} with an approximation given by Corollary~\ref{cor:tikhonov_fr_approx}, ignoring the negative term.
Next, we upper-bound $\epsilon(\lambda \|\bw^{\star}\|^2 + \lambda \|\bbeta^{\star}\|^2+ \Riskh(\htrg_{\bw^\star, \bbeta^\star})) + \lambda \|\bw^{\star}\|^2 + \lambda \|\bbeta^{\star}\|^2$ by Lemma~\ref{lem:norm_bound_opt}.
\end{proof}

The following proposition is used to derive the \verb!GreedyTL! in Section~\ref{sec:derivation}.
\begin{prop}
  \label{prop:rls_acc}
  Define the regularized accuracy as,
  \[
  \Arh(\bw) := 1 - \left(\frac{1}{m} \|\bX\tp \bw - \by\|_2^2 + \lambda \|\bw\|_2^2\right)~.
  \]
  We are given $\bX \in \reals^{n \times m}$, $\by \in \reals^m$, $S \subseteq \{1, \ldots, n\}$, and $\lambda \in \reals^+$.
  Furthermore, assume that $\frac{\|\by\|_2^2}{m} = 1$, and let
  $\bhatX$ be the submatrix of $\bX$, selecting rows indexed by $S$.
  Then we have that,
  \begin{align}
    \max_{\bw, \supp(\bw) = S} \left\{\Arh(\bw)\right\} &= \frac{1}{m}\by\tp \bhatX\tp (\bhatX \bhatX\tp + m \lambda \bI)^{-1} \bhatX \by \label{eq:A_emp_primal}\\
    &= \frac{1}{m} \by\tp (\bhatX\tp \bhatX + m \lambda \bI)^{-1} \bhatX\tp \bhatX \by~\label{eq:A_emp_dual}.
  \end{align}
  \begin{proof}
    Expanding the $\|\cdot\|^2$ in $\Arh(\bw)$ and using the fact that $\frac{\|\by\|^2}{m} = 1$, gives us
    \[
    \Arh(\bw) = \frac{2}{m} \bw\tp \bhatX \by - \frac{1}{m} \bw\tp (\bhatX \bhatX\tp + m \lambda \bI) \bw~.
    \]
    Now we have that $\frac{\partial \Arh(\bw)}{\partial \bw} = 0 \Rightarrow \bw = (\bhatX \bhatX\tp + m \lambda \bI)^{-1} \bhatX \by$.
    Denote $\bG = (\bhatX \bhatX\tp + m \lambda \bI)^{-1}$ and set optimal solution $\bw^\star = \bG \bhatX \by$.
    By putting $\bw^\star$ into the objective we have,
    \begin{align*}
    \Arh(\bw^\star) &= \frac{2}{m} \by\tp \bhatX\tp \bG\tp \bhatX \by - \frac{1}{m} \by\tp \bhatX\tp \bG\tp \bG^{-1} \bG \bhatX \by\\
                   &= \frac{1}{m} \by\tp \bhatX\tp \bG\tp \bhatX \by~.
    \end{align*}
    This proves the first statement.

    Now we turn to the second statement, that is solution in the dual variables.
    By using dual variable identity $(\bhatX \bhatX\tp + m \lambda \bI)^{-1} \bhatX = \bhatX (\bhatX\tp \bhatX + m \lambda \bI)^{-1}$~\cite{mohri2012foundations}, we write solution w.r.t. $\bw$ as
    $\bw = \bhatX (\bhatX\tp \bhatX + m \lambda \bI)^{-1} \by$.
    Denoting $\bG = (\bhatX\tp \bhatX + m \lambda \bI)^{-1}$, setting optimal solution $\bw^\star = \bhatX \bG \by$, and putting $\bw^\star$ into the objective we have,
    \begin{align*}
      &\Arh(\bw^\star) = \frac{2}{m} \by\tp \bG\tp \bhatX\tp \bhatX \by\\
      &- \frac{1}{m} \by\tp \bG\tp \bhatX\tp (\bhatX \bhatX\tp + m \lambda \bI) \bhatX \bG \by
                       = \frac{1}{m} \by\tp \bG \bhatX\tp \bhatX \by~.
    \end{align*}
    The last fact comes from the observation that $\bhatX \bG = (\bhatX \bhatX\tp + m \lambda \bI)^{-1} \bhatX$ by dual variable identity.
    This concludes the proof of the second statement.
  \end{proof}
\end{prop}

\end{document}